\documentclass[letterpaper]{article}
\usepackage[preprint]{aaai2027}
\usepackage[hyphens]{url}
\usepackage{graphicx}
\usepackage{natbib}
\usepackage{caption}
\usepackage{soul}
\usepackage{amsmath}
\usepackage{amsthm}
\usepackage{amssymb}
\usepackage{booktabs}
\usepackage{algorithm}
\usepackage{algorithmic}
\usepackage{multirow}
\usepackage{array}
\usepackage{tikz}

\newtheorem{theorem}{Theorem}
\newtheorem{lemma}{Lemma}
\newif\ifwithappendix
\withappendixtrue 

\title{Diffusion-Guided Uncertainty-Aware Delayed Policy Optimization}

\author{
    Junqi Tu\textsuperscript{\rm 1},
    Zejiao Liu\textsuperscript{\rm 2},
    Fangfei Li\textsuperscript{\rm 2}\corresponding,
    Yang Tang\textsuperscript{\rm 1}\corresponding
}
\affiliations
{
    \textsuperscript{\rm 1}The Key Laboratory of Smart Manufacturing in Energy Chemical Process, Ministry of Education,\\
    East China University of Science and Technology, Shanghai, China\\
    \textsuperscript{\rm 2}The School of Mathematics, East China University of Science and Technology, Shanghai, China\\
    \{23012389,liuzejiao\}@mail.ecust.edu.cn,
    \{lifangfei,yangtang\}@ecust.edu.cn
}

\begin{document}

\maketitle

\begin{abstract}
Reinforcement learning in real world environments often suffers from severe performance degradation due to delayed feedback. Existing approaches typically mitigate performance degradation caused by observation delays by constructing augmented states or predicting the true states. However, these methods often overlook the inherent discrepancy between delayed state and true states induced by stochastic MDP. We theoretically prove the existence of such a discrepancy and show that it leads to the degradation of the optimal policy. To address this challenge, we propose Diffusion Guided Uncertainty Aware Delayed Policy Optimization (DUPO). Our method explicitly models the relationship between delayed state message and the current state using a diffusion model, and leverages the resulting discrepancy estimates to weight delayed policies. Extensive experiments on continuous robotic control tasks with multiple stochastic delays demonstrate that DUPO consistently outperforms existing methods and remains effective even under long and random delay scenarios.

\end{abstract}

\section{Introduction}
Reinforcement learning (RL) has recently achieved remarkable progress in a variety of complex decision-making domains, including Go \citep{Silver2018AlphaZero}, competitive multiplayer online games \citep{Berner2019Dota2}, continuous control and simulated robotics \citep{Haarnoja2018SAC}, as well as autonomous driving related tasks \citep{Feng2023DenseDRL}. 
Despite these successes, an implicit assumption is often made that agent--environment interactions are instantaneous, such that actions take effect immediately and observations faithfully reflect the current state.
In real-world systems, however, this assumption frequently breaks down due to communication latency, sensor pipelines, and the computational overhead of model inference. 
Without properly modeling and compensating for delays, RL agents may suffer severe degradation and even safety-critical failures \citep{DulacArnold2021ChallengesRealWorldRL}; for instance, substantial performance loss in trading \citep{HasbrouckSaar2013LowLatencyTrading}, instability in dynamical systems \citep{DugardVerriest1998TimeDelay,GuNiculescu2003AdvancesTimeDelay}, and reduced training efficiency and reproducibility in real-world robotics \citep{Mahmood2018SettingUpRLRealWorldRobot}. 
Therefore, explicitly incorporating action and observation delays into RL algorithm design is a crucial step towards reliable real-world deployment. 
Existing studies on delayed RL can be broadly categorized into three main lines. 
Memoryless methods learn policies that depend only on the most recent observation, thereby ignoring the non-Markovian nature induced by observation delays \citep{Schuitema2010ControlDelay}. 
While simple to implement, these methods typically incur significant performance drops under large delays. 
Augmentation-based methods restore the Markov property by constructing an information state that concatenates the last observed state with the sequence of actions executed thereafter, reducing delayed RL to an augmented-state delayed MDP \citep{AltmanNain1992ClosedLoopDelayedInformation,KatsikopoulosEngelbrecht2003MDPDelays,Bertsekas2012DynamicProgramming}, on which classical RL algorithms can be applied \citep{Nath2021RevisitingStateAugmentation}. 
However, as the delay horizon increases, the augmented state space expands rapidly, resulting in substantially higher sample complexity and pronounced instability in TD-style learning under long delays, reflecting a classic manifestation of the curse of dimensionality.
Model-based methods seek to avoid learning directly in the rapidly expanding augmented space by instead computing compact statistics  about the unknown current state from history and using them as policy inputs \citep{Walsh2009DelayedFeedback,Firoiu2018ActionDelay,AgarwalAggarwal2021DelayedObservations,Derman2021DelayedEnvironmentsNonStationary,Liotet2021DTRPO}. 
In addition, some works accelerate learning under long delays by introducing auxiliary short-delay tasks; for example, boosting delayed RL via auxiliary short delays can guide learning in the long-delay task and improve sample efficiency \citep{Wu2024AuxiliaryShortDelays}.\\
\noindent Meanwhile, in real-world systems, delay durations are often stochastic rather than fixed. In random-delay environments, the discrepancy between delayed observations and the true current state generally grows with the delay length, causing the optimal policy to vary across different delay regimes. Ignoring such delay heterogeneity may lead the agent to behave overly conservatively under short delays or excessively aggressively under long delays, resulting in degraded performance and unstable training dynamics.
In random MDPs, we theoretically establish the existence of delay-induced state discrepancy and show that it amplifies as the delay length increases. We further demonstrate that approaches based on single-point state prediction inevitably suffer performance degradation, as they fail to capture this constantly changing discrepancy.
To explicitly account for delay-induced state discrepancy in delayed reinforcement learning, we propose Diffusion-Guided Uncertainty-Aware Delayed Policy Optimization (DUPO). The key idea is to model the mismatch between delayed state message and the underlying true state during policy optimization. DUPO leverages a generative diffusion model to characterize the conditional relationship between delayed observations and the current true state, enabling a principled representation of potentially multi-modal discrepancy structures induced by varying delay lengths. Based on these multi-modal state representations, DUPO estimates the uncertainty of the action--value function conditioned on delayed observations and uses this uncertainty to adaptively weight delayed policy updates. This uncertainty-aware weighting yields delay-adaptive policies that adjust their behavior across delay regimes, leading to more stable and sample-efficient learning in random-delay environments.\\
\noindent We transform the originally deterministic continuous-control MuJoCo tasks into stochastic MDPs by injecting Gaussian noise into the action execution, and further introduce random observation delays to emulate realistic delayed-feedback settings. Across a wide range of delay configurations, DUPO consistently demonstrates superior performance and improved training stability compared to other strong baselines designed for stochastic delayed environments. These results indicate that DUPO is not only effective under random delays but also robust to the compounded challenges arising from environmental stochasticity and delay variability. Our contributions can be summarized as follows:
\begin{itemize}
    \item We theoretically prove the existence of delay-induced state discrepancy and show that it amplifies with increasing delay length, and further demonstrate that single-point state prediction methods inevitably suffer from performance degradation.
    \item We model the conditional distribution between delayed state message and the true current state using a generative diffusion model instead of single-point state prediction, capturing multi-modal discrepancy structures induced by random delays.
    \item We estimate action-value uncertainty under delayed observations and leverage it to adaptively weight delayed policy updates, enabling delay-adaptive policy learning.
    \item DUPO consistently outperforms state-of-the-art baselines across multiple random delay settings.
\end{itemize}

\begin{figure}
    \centering
    \includegraphics[width=1\linewidth]{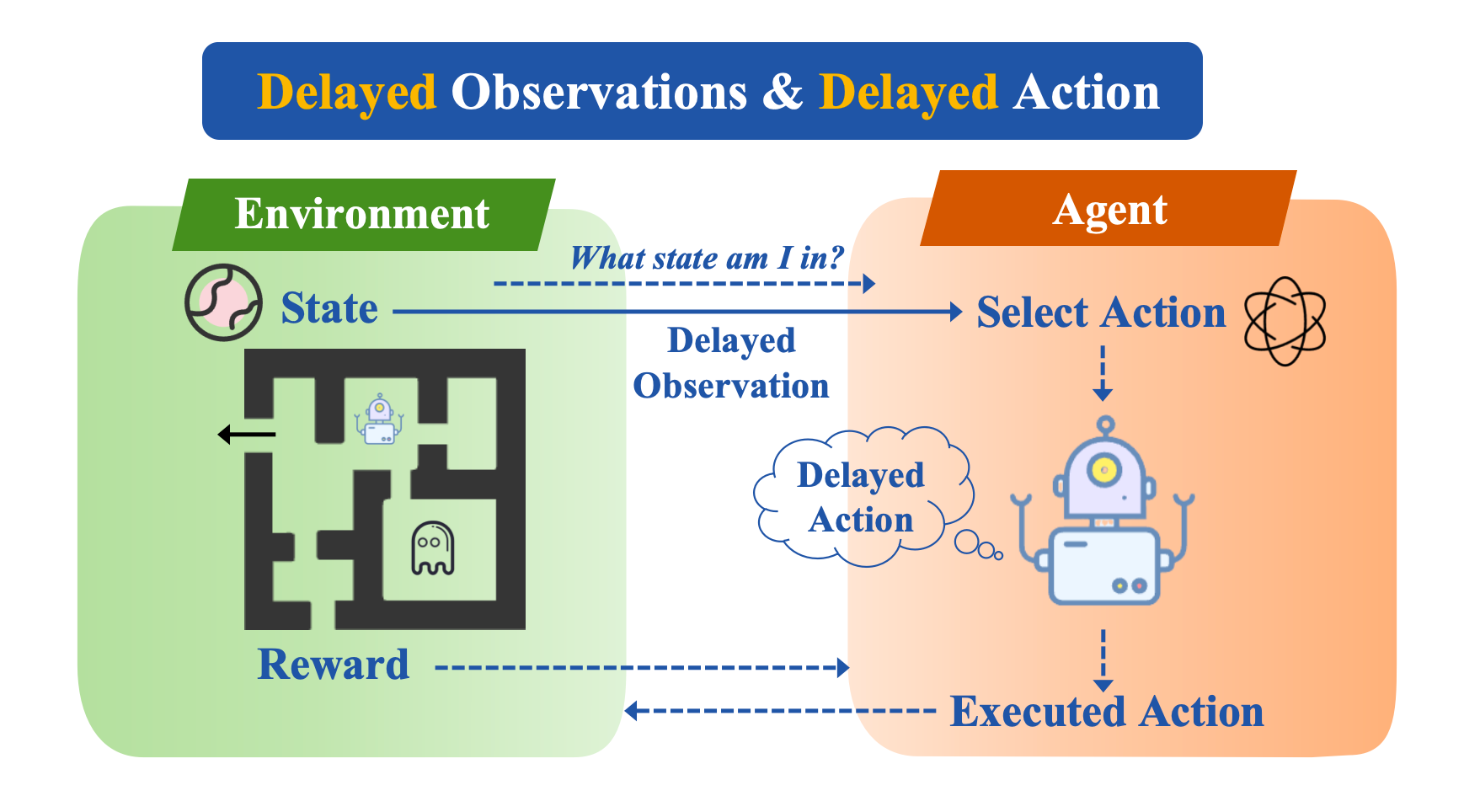}
    \caption{The agent faces observation delays and action execution delays.}
    \label{fig:delay_feedback}
\end{figure}
\section{Related Work}
\label{sec:related_work}

\paragraph{Reinforcement learning with delayed feedback.}
Existing studies on delayed reinforcement learning can be broadly categorized into three main lines. 
Memoryless methods learn policies that depend only on the most recent observation, thereby ignoring the non-Markovian nature induced by observation delays \citep{Schuitema2010ControlDelay}. 
While simple to implement, these methods typically suffer from severe performance degradation as the delay length increases. \\
\noindent Augmentation-based methods restore the Markov property by constructing an information state that concatenates the last observed state with the sequence of subsequently executed actions, reducing delayed reinforcement learning to an augmented-state MDP \citep{AltmanNain1992ClosedLoopDelayedInformation,KatsikopoulosEngelbrecht2003MDPDelays,Bertsekas2012DynamicProgramming}. 
This formulation enables the direct application of standard RL algorithms \citep{Nath2021RevisitingStateAugmentation}. 
However, the augmented state space grows rapidly with the delay horizon, leading to substantially increased sample complexity and pronounced instability in TD-style learning under long delays, reflecting a classic manifestation of the curse of dimensionality. \\
\noindent Model-based and belief-based methods seek to avoid learning directly in the rapidly expanding augmented space by instead estimating compact statistics or distributions of the unknown current state from historical information and using them as policy inputs \citep{Walsh2009DelayedFeedback,Firoiu2018ActionDelay,AgarwalAggarwal2021DelayedObservations,Derman2021DelayedEnvironmentsNonStationary,Liotet2021DTRPO}. 
Recent advances further incorporate principled probabilistic modeling and variational inference to address delayed decision making; for example, Variational Delayed Policy Optimization formulates delayed RL from a variational perspective and optimizes policies under delayed information constraints \citep{Wu2024VDPO}, while other works analyze and mitigate signal delays in deep RL through delay-aware architectures and training strategies \citep{wang2024addressing}. 
In addition, several approaches accelerate learning under long delays by introducing auxiliary short-delay tasks; leveraging such auxiliary delays has been shown to effectively guide policy learning in long-delay settings and improve sample efficiency \citep{Wu2024AuxiliaryShortDelays}.
\paragraph{Diffusion models and uncertainty-aware reinforcement learning.}
Generative modeling has recently become a powerful tool in reinforcement learning, enabling expressive and potentially multi-modal representations of policies, dynamics, and latent states \citep{Ding2024QVPO}. 
Diffusion models, in particular, are well suited for learning complex conditional distributions and have been used to improve robustness against observation perturbations in offline and model-based RL \citep{Yang2024DMBP}. 
In parallel, uncertainty-aware value estimation and policy optimization, such as ensemble-based critics and uncertainty-weighted updates, have been shown to improve robustness under partial observability and distribution shift \citep{Wu2021UWAC}. 
Building on these advances, we leverage diffusion modeling to explicitly capture the multi-modal state discrepancy induced by random observation delays and incorporate the resulting uncertainty into policy optimization, yielding delay-regime-adaptive behavior in delayed RL environments.

\section{Preliminaries}
\label{sec:prelim}

A Partially Observable Markov Decision Process (POMDP)   \citep{kaelbling1998planning} can be defined by a tuple $\mathcal{M}=(\mathcal{S},\mathcal{A},T,r,\Omega,O,\gamma)$, where $\mathcal{S}$, $\mathcal{A}$, and $\Omega$ denote the state, action, and observation spaces. The transition dynamics are defined by  $T(s'\mid s,a)$,the reward function by $r(s,a)$, and the observation probabilities by $O(o\mid s)$. At each timestep $t$, the environment is in a latent state $s_t\in\mathcal{S}$, the agent receives an observation $o_t\sim O(\cdot\mid s_t)$, selects an action $a_t\in\mathcal{A}$ according to a policy $\pi$, obtains a reward $r_t=r(s_t,a_t)$, and the environment transitions to the next state according to $s_{t+1}\sim T(\cdot\mid s_t,a_t)$. The agent repeats this process to find an optimal policy $\pi^*$ that maximizes the expected discounted cumulative reward (i.e., return) over a finite or infinite horizon $H$, which is given by
\begin{equation}
\pi^* = \arg\max_{\pi} \; \mathbb{E} \left[ \sum_{k=0}^{H-1} \gamma^{k} r_{k+1} \;\middle|\; \pi, d_0 \right].
\label{eq:optimal_policy}
\end{equation}
where $d_0$ denotes the initial state distribution.
 In the fully observable case, $o_t=s_t$ almost surely, and the process reduces to a standard Markov Decision Process (MDP).

\subsection{Random Observation Delay POMDP}
\label{subsec:random_delay_mdp}

We model \emph{random observation delay} by augmenting the above POMDP with a random delay process.
Formally, a random observation delay environment is defined as a pair
\(\langle \mathcal{M}, \mathcal{D} \rangle\),
where \(\mathcal{M}\) is a POMDP and \(\mathcal{D}\) is a probability distribution over non-negative integers,
representing random observation delays.
At each timestep \(t\), instead of observing the current state \(s_t\),
the agent receives a delayed state observation corresponding to the latent state
\(s_{t-\Delta T}\), where the delay \(\Delta T \sim \mathcal{D}\).
In addition to the delayed state, the agent has access to the sequence of actions
\(a_{t-\Delta T:t-1}\) executed since the observation was generated.
Together, the delayed state and the history actions form a
\emph{delayed state message}, denoted by $M_t = \big(s_{t-\Delta T},\, a_{t-\Delta T:t-1}\big)$
which summarizes all information available to the agent at time \(t\)
and serves as the input for decision making.
Throughout this paper, we assume that the delay is finite and bounded by a maximum delay \(D\).

\subsection{Conditional Diffusion Model}
\label{subsec:conditional_diffusion_prelim}

A conditional diffusion model is a generative model that learns a conditional data distribution
$p_\theta(x_0 \mid c)$, where $x_0 \in \mathbb{R}^d$ denotes a target variable and $c$ denotes conditional information.
It introduces a sequence of latent variables $\{x_k\}_{k=1}^{K}$ by gradually perturbing $x_0$ with Gaussian noise
according to a fixed variance schedule $\{\beta_k\}_{k=1}^{K}$:
\begin{equation}
q(x_k \mid x_{k-1}) = \mathcal{N}\!\left(\sqrt{1-\beta_k}\,x_{k-1},\,\beta_k I\right), \quad k=1,\dots,K.
\end{equation}
Let $\bar{\alpha}_k=\prod_{i=1}^{k}(1-\beta_i)$, then the marginal at step $k$ admits the closed form \citep{DBLP:conf/nips/HoJA20} 
\begin{equation}
x_k=\sqrt{\bar{\alpha}_k}\,x_0+\sqrt{1-\bar{\alpha}_k}\,\epsilon,\qquad \epsilon\sim\mathcal{N}(0,I).
\end{equation}
The reverse generative process is parameterized by a conditional noise predictor $\epsilon_\theta(x_k,k,c)$
and is trained with the standard noise-prediction objective
\begin{equation}
\min_\theta\;\mathbb{E}_{x_0,k,\epsilon}\Big[\big\|\epsilon-\epsilon_\theta(x_k,k,c)\big\|_2^2\Big],
\end{equation}
where $k$ is sampled uniformly from $\{1,\dots,K\}$.
At inference time, given conditioning $c$, a sample $x_0 \sim p_\theta(\cdot \mid c)$ is obtained by starting from
$x_K \sim \mathcal{N}(0,I)$ and iteratively applying the learned reverse denoising transitions.
This framework provides a flexible approximation to complex conditional distributions and naturally represents
uncertainty through random sampling.

\begin{figure*}[t]
    \centering
    \includegraphics[width=1\linewidth]{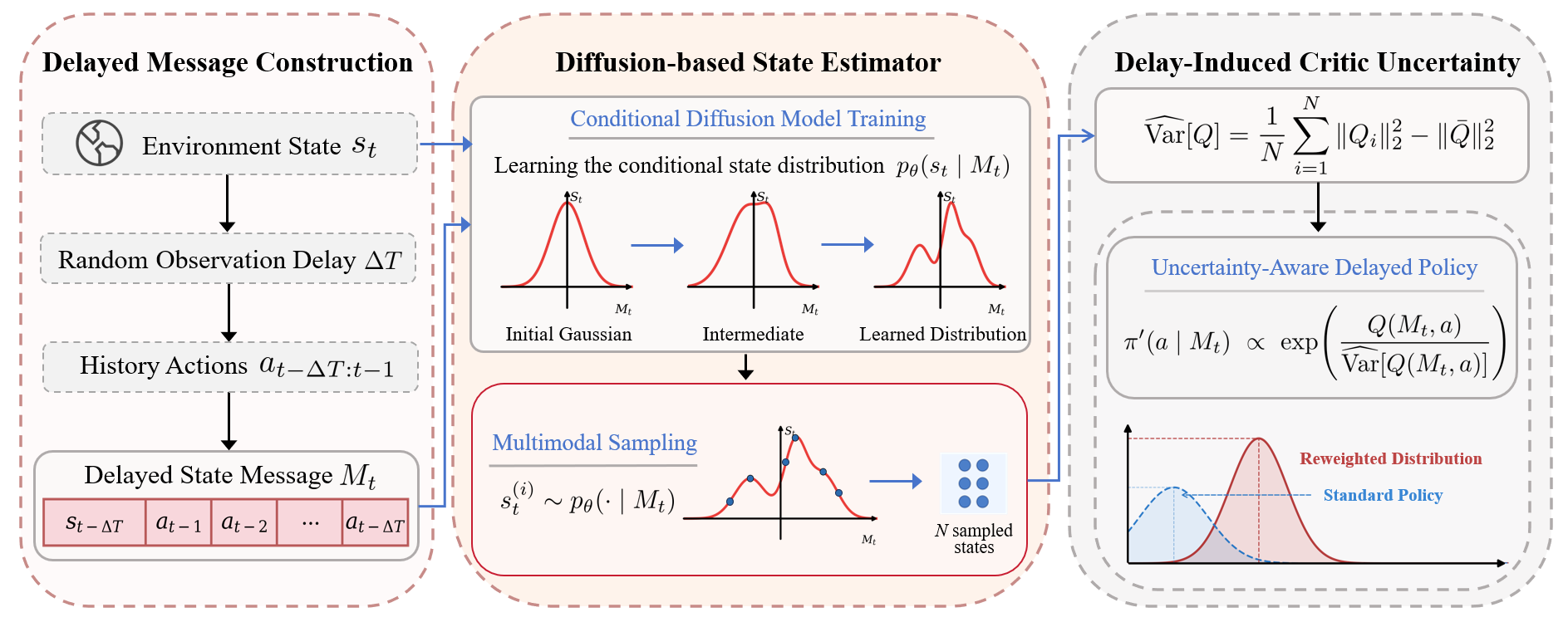}
\caption{Overview of DUPO. Left: we construct a delayed state message $M_t$ by combining the delayed observation $s_{t-\Delta T}$ with the intervening action history. Middle: a conditional diffusion model learns the posterior $p_\theta(s_t \mid M_t)$ and generates multi-modal samples of the delay-free state. Right: critic uncertainty induced by state discrepancy is estimated from these samples and used to uncertainty-weight policy updates, yielding delay-adaptive weighted action distributions.}

    \label{fig:dupo_overview}
\end{figure*}

\section{Our Approach}
In this section, we introduce Diffusion-Guided Uncertainty-Aware Delayed Policy Optimization (DUPO). DUPO first learns a conditional diffusion model to represent the posterior over the latent delay-free state from delayed state messages, and then we leverage the posterior distribution to characterize the state discrepancy induced by observation delays and quantify the resulting uncertainty in action value estimation. We then weight candidate actions according to their value uncertainty to obtain a more stable delayed policy. Finally, by integrating this uncertainty-weighted delayed policy into Soft Actor-Critic, we arrive at our DUPO algorithm.
\subsection{Discrepancy under Random Delays}
\label{subsec:discrepancy}

We study a stochastic delayed MDP with latent dynamics
$s_{t+1}=f(s_t,a_t)+\xi_t$ and delayed message $M_t=(s_{t-\Delta T},a_{t-\Delta T:t-1})$.
Random delays and system noise induce an inherent mismatch between $M_t$ and the current state $s_t$. The bounds below are stated conditional on a realized delay value $\Delta T$; averaging over a random delay distribution yields the corresponding bounds with the last terms averaged over $\Delta T$. The detailed proof is provided in the supplementary material.

\begin{theorem}[Irreducible and delay-amplified state discrepancy]
\label{thm:irreducible_delay_discrepancy}
Let $\Sigma(M_t):=\mathrm{Cov}(s_t\mid M_t)$.
Under mild regularity conditions, any measurable estimator $g(M_t)$ satisfies
\begin{align}
\mathbb{E}\|s_t-g(M_t)\|_2^2
\;\ge\;
\mathbb{E}\,\mathrm{tr}\!\left(\Sigma(M_t)\right)
\;\ge\;
\nu \sum_{i=0}^{\Delta T-1}\kappa^i
\;\ge\;
\nu \Delta T.
\label{eq:irreducible_chain}
\end{align}
where $\nu>0$ is the intrinsic noise level and $\kappa\ge1$ captures uncertainty amplification.
\end{theorem}
\noindent A common workaround is to compensate delays by a point prediction $\hat s_t=g(M_t)$ and act greedily on $\hat s_t$.
The next result shows that such point-estimate methods suffer a delay-amplified performance loss.
\begin{theorem}[Delay-amplified suboptimality of point-estimate methods]
\label{thm:point_estimate_short}
Consider any point estimator $\hat s_t=g(M_t)$ and the induced policy $a_t=a^\star(\hat s_t)$, where
$a^\star(s)=\arg\max_a Q^\star(s,a)$.
If $Q^\star(s,\cdot)$ is $\mu$-strongly concave in $a$ and $a^\star(\cdot)$ is locally inverse-Lipschitz with constant $L_->0$,
then
\begin{align}
J(\pi^\star)-J(\pi_g)\ \ge\ \frac{\mu L_-^2 \nu}{2(1-\gamma)}\,\Delta T.
\label{eq:point_estimate_short}
\end{align}
\end{theorem}
\noindent Together, these theorems motivate modeling the full posterior density $p(s_t\mid M_t)$, rather than a single prediction, to properly account for delay-induced uncertainty.

\subsection{Diffusion-based State Estimator}
\label{subsec:dbde}

To model the relationship between delayed state message and the delay-free state, we employ a conditional diffusion model with strong expressive capacity.
Specifically, the model captures the conditional distribution $p_\theta(s_t \mid M_t)$
and provides a flexible representation of the potentially multi-modal posterior distribution induced by random observation delays \citep{DBLP:conf/nips/HoJA20}.\\
\noindent The diffusion model introduces a sequence of latent variables $\{s_t^n\}_{n=0}^N$ with $s_t^0 = s_t$, obtained by progressively corrupting the clean state with Gaussian noise according to a predefined variance schedule.
The forward noising process is defined by the Markov transition
\begin{equation}
q(s_t^{n} \mid s_t^{n-1})
= \mathcal{N}\!\left(
\sqrt{\alpha_n}\, s_t^{n-1},
(1-\alpha_n)I
\right), \quad n = 1,\dots,N.
\label{eq:dbde_forward}
\end{equation}
where $\alpha_n\in (0,1), \forall n=1, \cdots, N$ controls the noise level at each diffusion step.
By composing these transitions, the marginal distribution at step $n$ admits the closed form
\begin{equation}
q(s_t^{n} \mid s_t)
= \mathcal{N}\!\left(
\sqrt{\bar{\alpha}_n}\, s_t,
(1-\bar{\alpha}_n)I
\right),
\qquad
\bar{\alpha}_n = \prod_{i=1}^{n} \alpha_i.
\label{eq:dbde_forward_closed}
\end{equation}
which enables efficient sampling of noisy states at arbitrary diffusion steps.\\
The reverse generative process is parameterized as a conditional Markov chain that inverts the noising process given delayed message $M_t$ \citep{DBLP:conf/nips/DhariwalN21},
\begin{equation}
p_\theta(s_t^{n-1} \mid s_t^{n}, M_t)
= \mathcal{N}\!\left(
\mu_\theta(s_t^{n}, M_t, n),
\sigma_n^2 
\right).
\label{eq:dbde_reverse}
\end{equation}
where the mean $\mu_\theta$ is computed via a neural noise predictor $\epsilon_\theta(s_t^{n}, M_t, n)$ using the standard $\epsilon$-parameterization,
\begin{equation}
\mu_\theta(s_t^{n}, M_t, n)
= \frac{1}{\sqrt{\alpha_n}}
\left(
s_t^{n}
- \frac{1-\alpha_n}{\sqrt{1-\bar{\alpha}_n}}
\epsilon_\theta(s_t^{n}, M_t, n)
\right).
\label{eq:dbde_mean}
\end{equation}
Conditioning on $M_t$ allows the denoising process to adapt to different delay realizations and action histories.\\
Model parameters are learned by minimizing the standard noise-prediction objective.
Specifically, a diffusion step $n$ is sampled uniformly from $\{1,\dots,N\}$, Gaussian noise $\epsilon \sim \mathcal{N}(0,I)$ is drawn, and a noisy state is constructed via
\begin{equation}
s_t^{n}
= \sqrt{\bar{\alpha}_n}\, s_t
+ \sqrt{1-\bar{\alpha}_n}\, \epsilon.
\end{equation}
The training loss is then given by
\begin{equation}
\mathcal{L}_{\mathrm{diff}}
= \mathbb{E}_{s_t, M_t, \epsilon, n}
\left[
\left\|
\epsilon - \epsilon_\theta(s_t^{n}, M_t, n)
\right\|_2^2
\right].
\label{eq:dbde_loss}
\end{equation}\\
After training, the learned reverse process defines an implicit conditional generative model for the delay-free state.
Given delayed message $M_t$, samples from $p_\theta(s_t \mid M_t)$ are obtained by starting from $s_t^N \sim \mathcal{N}(0,I)$ and iteratively applying the reverse transitions in Eq.~\eqref{eq:dbde_reverse}.
This corresponds to the conditional density
\begin{equation}
p_\theta(s_t \mid M_t)
= \int
p_\theta(s_t^{0:N-1} \mid s_t^{N}, M_t)\,
p(s_t^{N})\,
\mathrm{d}s_t^{1:N}.
\label{eq:dbde_conditional}
\end{equation}
where $p(s_t^{N}) = \mathcal{N}(0,I)$.
In our framework, multiple samples from this distribution are used to characterize the discrepancy between delayed state message $M_t$ and the delay-free state $s_t$, which will be exploited by subsequent uncertainty-aware policy optimization.

\setlength{\tabcolsep}{5pt}

\begin{table*}[ht]
\centering
\small
\caption{Normalized performance on MuJoCo-v4 under random observation delays $\Delta T_{\max}\in\{5,10,25\}$.
Each entry reports the final return after training for $1$M global environment steps, averaged over random seeds, for four algorithms (DC/AC, DUPO, State-Pred, and State-Aug).
For each $(\Delta T_{\max},\text{env})$, we normalize by the DC/AC baseline:
$\tilde{\mu}=\mu/\mu_{\text{DC/AC}}$ and $\tilde{\sigma}=\sigma/\mu_{\text{DC/AC}}$.
Blue indicates the best normalized mean within each $\Delta T_{\max}$.}

\label{tab:mujoco_v4_delay_norm}
\begin{tabular}{c l *{5}{c}}
\toprule
\multicolumn{2}{c}{Environment}
& Ant-v4 & HalfCheetah-v4 & Hopper-v4 & Walker2d-v4 & Swimmer-v4 \\
\cmidrule(lr){1-2}
$\Delta T_{\max}$ & Algorithm &  &  &  &  &  \\
\midrule

\multirow{4}{*}{5}
& DC/AC
& $1.00\pm0.10$ & $1.00\pm0.09$ & $1.00\pm0.12$ & $1.00\pm0.33$ & $1.00\pm0.43$ \\
& DUPO (ours)
& \textcolor{blue}{$2.43\pm0.11$} & \textcolor{blue}{$4.53\pm0.30$} & \textcolor{blue}{$7.24\pm0.41$} & \textcolor{blue}{$1.42\pm0.04$} & $1.25\pm0.45$ \\
& State Prediction
& $2.28\pm0.37$ & $3.93\pm0.21$ & $3.19\pm0.75$ & $0.28\pm0.02$ & \textcolor{blue}{$1.65\pm0.15$} \\
& State Augmentation
& $1.43\pm0.30$ & $3.78\pm0.78$ & $6.78\pm1.40$ & $1.19\pm0.37$ & $1.07\pm0.48$ \\
\midrule

\multirow{4}{*}{10}
& DC/AC
& $1.00\pm0.05$ & $1.00\pm0.03$ & $1.00\pm0.14$ & $1.00\pm0.64$ & $1.00\pm0.02$ \\
& DUPO (ours)
& $1.87\pm0.81$ & \textcolor{blue}{$7.53\pm0.15$} & \textcolor{blue}{$30.46\pm0.83$} & \textcolor{blue}{$1.62\pm0.64$} & \textcolor{blue}{$2.36\pm0.79$} \\
& State Prediction
& \textcolor{blue}{$1.90\pm0.35$} & $6.42\pm0.48$ & $6.41\pm3.06$ & $0.45\pm0.24$ & $1.75\pm0.72$ \\
& State Augmentation
& $1.30\pm0.16$ & $5.96\pm0.36$ & $23.71\pm4.24$ & $1.16\pm0.30$ & $2.15\pm0.64$ \\
\midrule

\multirow{4}{*}{25}
& DC/AC
& $1.00\pm0.05$ & $1.00\pm0.60$ & $1.00\pm0.04$ & $1.00\pm0.23$ & $1.00\pm0.02$ \\
& DUPO (ours)
& \textcolor{blue}{$2.14\pm0.92$} & \textcolor{blue}{$15.96\pm0.51$} & \textcolor{blue}{$10.34\pm3.44$} & \textcolor{blue}{$2.59\pm0.95$} & \textcolor{blue}{$2.98\pm0.10$} \\
& State Prediction
& $2.11\pm0.20$ & $13.19\pm0.66$ & $4.01\pm0.89$ & $0.66\pm0.08$ & $2.30\pm0.85$ \\
& State Augmentation
& $1.57\pm0.07$ & $5.10\pm0.95$ & $7.32\pm0.63$ & $0.82\pm0.07$ & $2.75\pm0.58$ \\
\bottomrule
\end{tabular}
\end{table*}
\subsection{Uncertainty-Weighted Delayed SAC}
\label{subsec:uwdsac}

To cope with random observation delays, we explicitly treat the current delay-free state as a latent
variable conditioned on the delayed message
$
M_t=\big(s_{t-\Delta T},\,a_{t-\Delta T:t-1}\big).
$
Using the trained conditional diffusion model, we draw $N$ plausible delay-free states
$
s_t^{(i)}\sim p_\theta(\cdot\mid M_t),
$
which yields a Monte Carlo approximation of the discrepancy induced by delay. This discrepancy is then
propagated to action value estimation: for any action $a$, we quantify the delay-induced critic uncertainty by
the sample variance
\begin{align}
\widehat{\mathrm{Var}}[Q(M_t,a)]
= \frac{1}{N}\sum_{i=1}^N Q_\psi(s_t^{(i)},a)^2
- Q_\psi(s_t,a)^2.
\end{align}
We transform this uncertainty into a reliability weight via inverse variance scaling 
\begin{equation}
    \omega(M_t,a)=\frac{\beta}{\widehat{\mathrm{Var}}\!\left[Q(M_t,a)\right]+\varepsilon}.
\end{equation}
and weight the base actor $\pi_\phi(\cdot\mid M_t)$ to obtain an uncertainty-weighted delayed policy~\citep{Wu2021UWAC}
\begin{equation}
\label{eq:pi_prime}
\pi'_\phi(a\mid M_t)=\frac{\omega(M_t,a)\,\pi_\phi(a\mid M_t)}
{\int \omega(M_t,\tilde a)\,\pi_\phi(\tilde a\mid M_t)\,d\tilde a}.
\end{equation}
Intuitively, actions whose values are highly sensitive to the latent state (hence more risky under delay)
receive smaller probability mass under $\pi'_\phi$, leading to a more conservative policy when the
delayed message is ambiguous.\\
We incorporate $\pi'_\phi$ into SAC while training the critic on delay-free transition tuples, following
the standard treatment in delayed RL~\citep{wang2024addressing}. Concretely,
\begin{align}
J_Q(\psi)
&=\mathbb{E}_{\mathcal{D}}\!\left[\big(Q_\psi(s_t,a_t)-y_t\big)^2\right],
\label{eq:critic_loss}\\
y_t
&= r_t + \gamma\,\mathbb{E}_{a_{t+1}\sim \pi'(\cdot \mid s_{t+1})}
\Big[\, Q_{\psi}(s_{t+1}, a_{t+1}) \notag\\
&\qquad - \alpha \log \pi'(a_{t+1}\mid s_{t+1}) \,\Big].
\label{eq:critic_target}
\end{align}
Since $s_{t+1}$ is delay-free, weighting is unnecessary and $\pi'(\cdot\mid s_{t+1})=\pi(\cdot\mid s_{t+1})$.\\
In practice, we optimize the delayed policy by updating the base actor $\pi_\phi$
under an uncertainty-weighted objective that preserves the effect of weighting.
Specifically, the actor is trained using a weighted SAC surrogate loss, where the
uncertainty weight $\omega(M_t,a)$ modulates the policy update.

\begin{align}
J_{\pi'}(\phi)
&=\mathbb{E}_{(M_t,s_t)\sim\mathcal{D}}
\Big[
\mathbb{E}_{a_t\sim\pi'_\phi(\cdot\mid M_t)}
\Big[ \notag\\
&\qquad
\alpha\log\pi'_\phi(a_t\mid M_t)-Q_\psi(s_t,a_t)
\Big]\Big]\\
&=\mathbb{E}_{(M_t,s_t)\sim\mathcal{D}}
\Big[
\mathbb{E}_{\mathbf{a}_t\sim\pi_\phi}
\Big[
\omega(M_t,a)\Big( \notag\\
&\qquad
\alpha\log\!\left(
\frac{\omega(M_t,a)\pi_\phi(\mathbf{a}_t\mid M_t)}
{\int_a \omega(M_t,a)\pi(a\mid M_t)\,da}
\right) \notag\\
&\qquad
-Q_\psi(s_t,\mathbf{a}_t)
\Big)\Big]\Big].
\label{eq:actor_obj}
\end{align}
This objective preserves the effect of uncertainty-aware weighting while remaining amenable to
standard SAC-style optimization.

\newcommand{\panelimg}[2]{%
\begin{tikzpicture}
  \node[inner sep=0pt, outer sep=0pt, anchor=north west] (im)
      {\includegraphics[width=\linewidth]{#1}};
  \node[anchor=north west, inner sep=0pt, outer sep=0pt,
        xshift=1.2pt, yshift=-1.2pt] at (im.north west)
      {\scriptsize\bfseries (#2)};
\end{tikzpicture}%
}
\begin{figure*}[ht]
    \centering

    \newcommand{\roww}{0.85\textwidth}
    \newcommand{\subw}{0.326\linewidth}
    \newcommand{\hgap}{\hspace{0.007\linewidth}}

    \newcommand{\subwB}{0.355\linewidth}
    \newcommand{\legw}{0.22\linewidth}
    \newcommand{\hgapB}{\hspace{0.010\linewidth}}

    \begin{minipage}{\roww}
        \centering
        \begin{minipage}[t]{\subw}
            \centering
            \panelimg{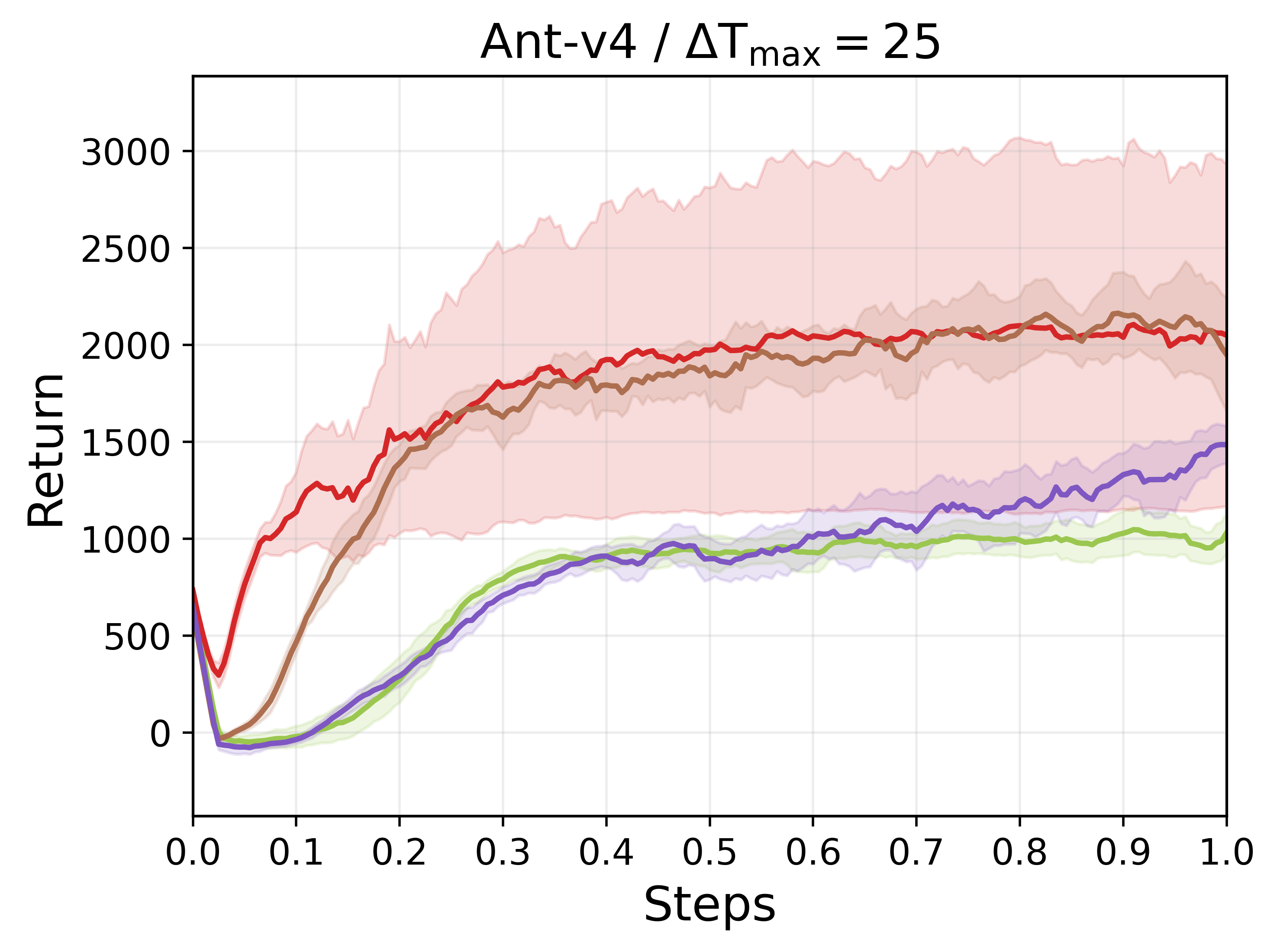}{a}
        \end{minipage}\hgap
        \begin{minipage}[t]{\subw}
            \centering
            \panelimg{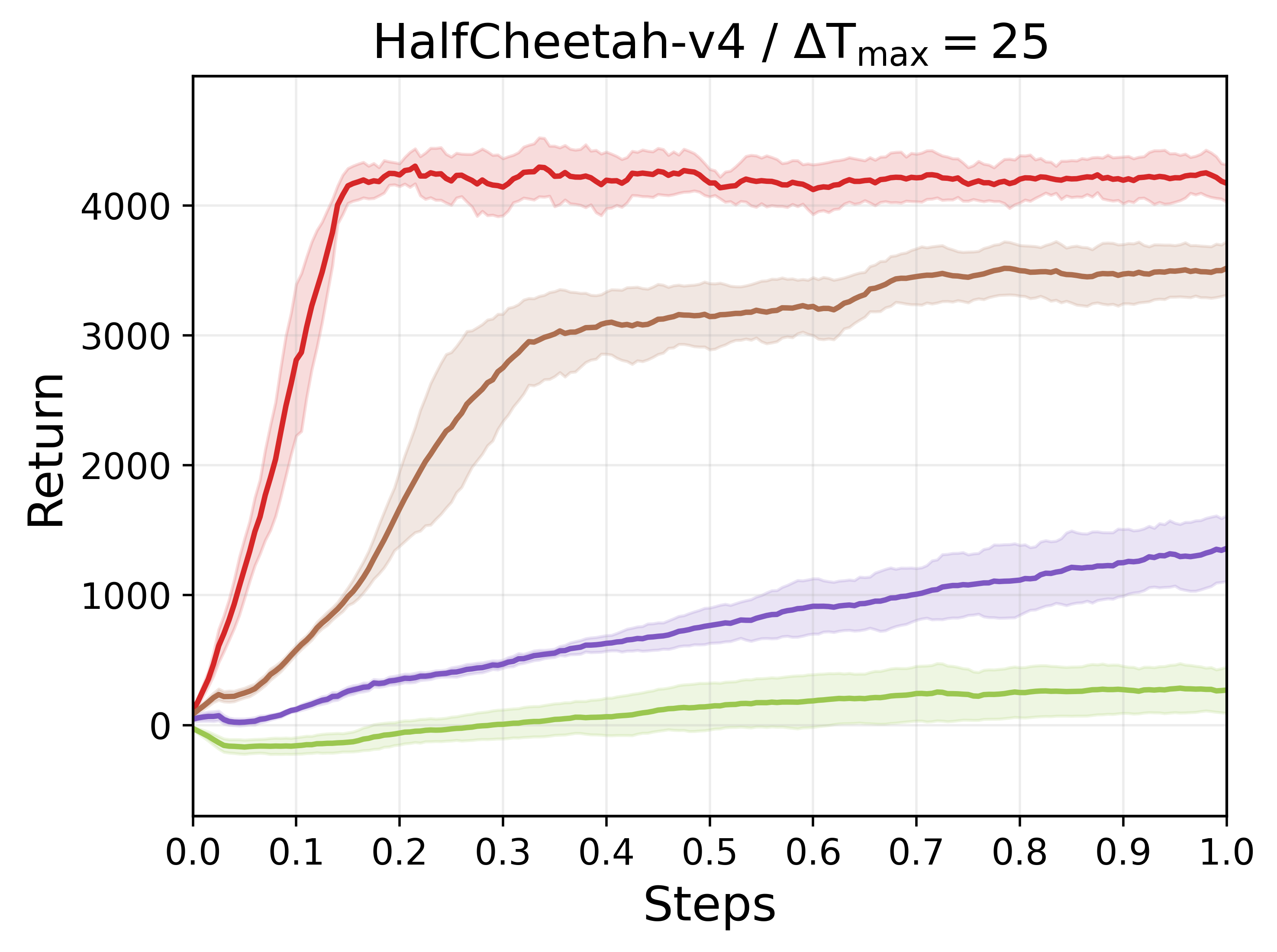}{b}
        \end{minipage}\hgap
        \begin{minipage}[t]{\subw}
            \centering
            \panelimg{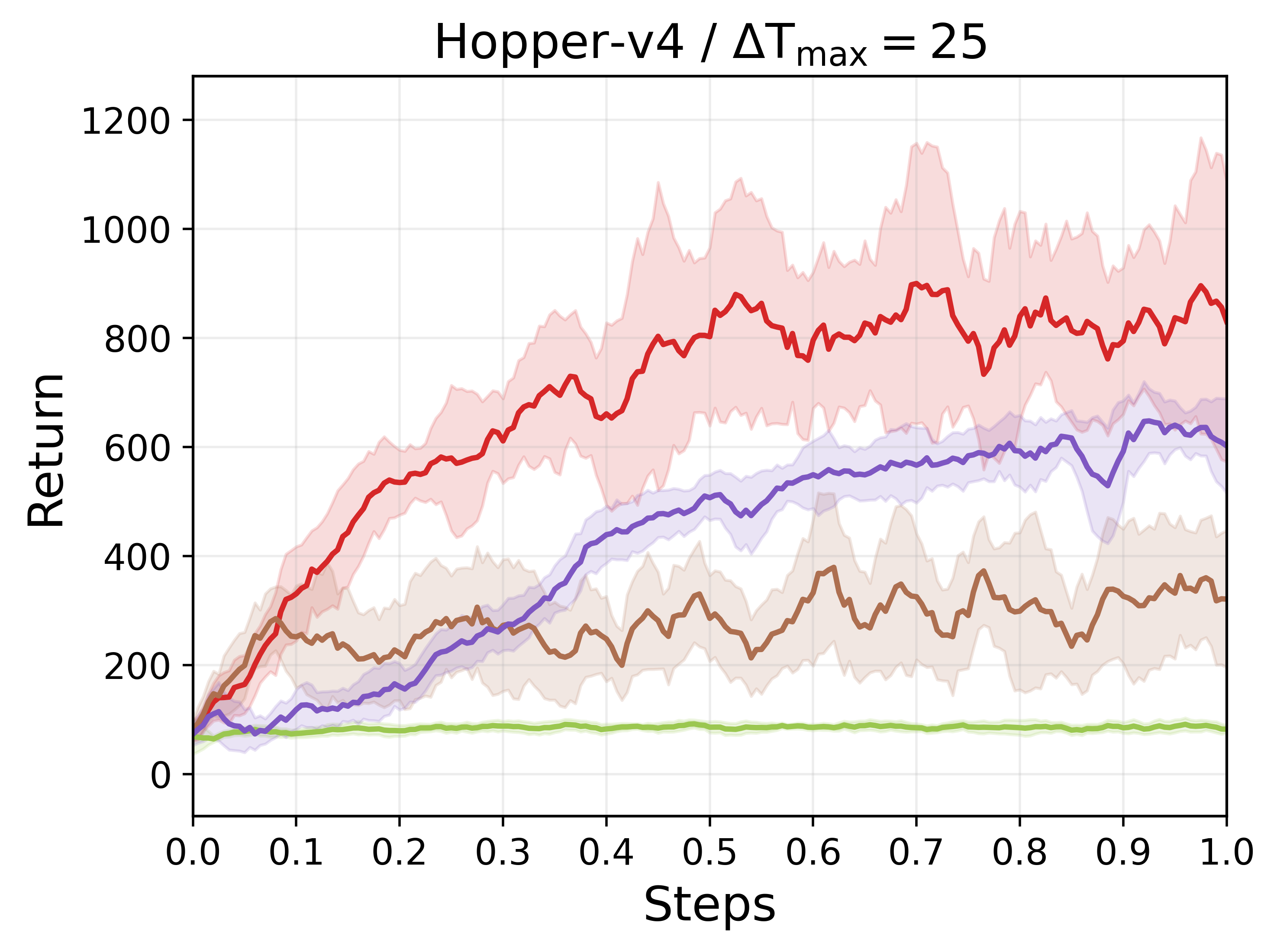}{c}
        \end{minipage}
    \end{minipage}

    \begin{minipage}{\roww}
        \centering
        \begin{minipage}[t]{\subwB}
            \centering
            \panelimg{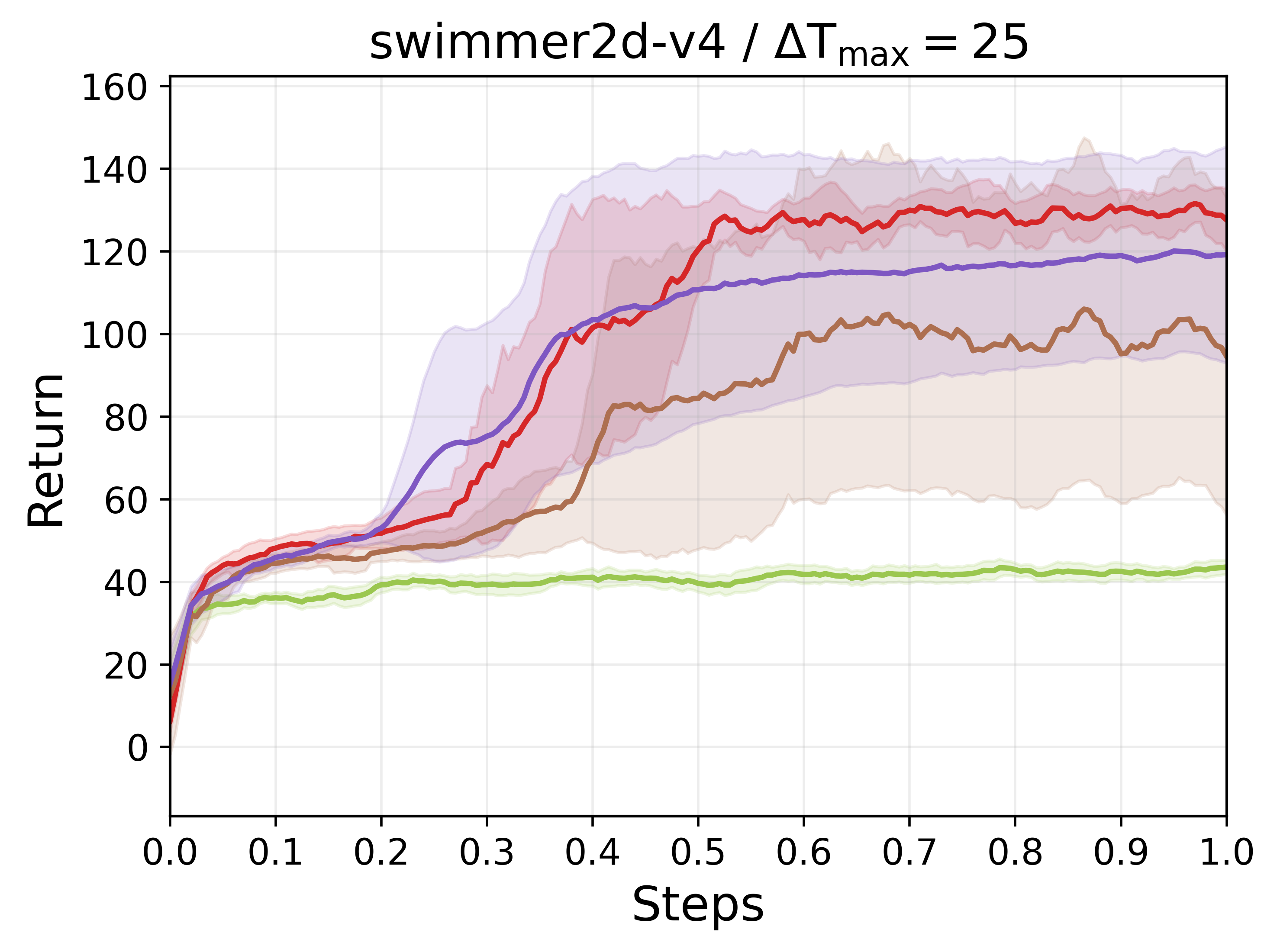}{d}
        \end{minipage}\hgapB
        \begin{minipage}[t]{\subwB}
            \centering
            \panelimg{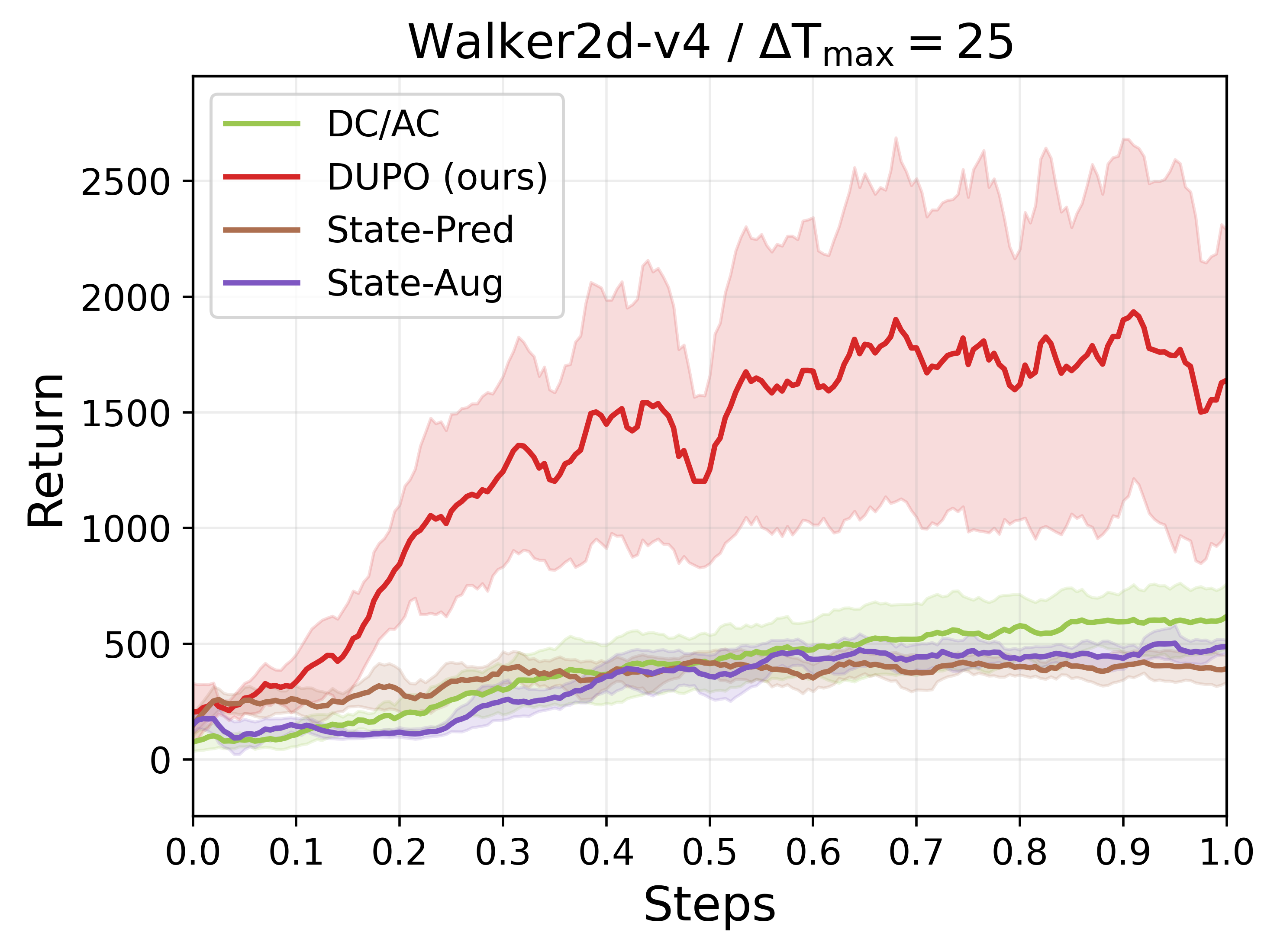}{e}
        \end{minipage}\hgapB
        
    \end{minipage}

    \caption{Learning curves on MuJoCo-v4 under long random observation delays ($\Delta T_{\max}=25$).
    Across five tasks, DUPO (ours) consistently achieves higher returns and faster convergence than
    baselines. Solid lines denote the mean over
    four seeds and shaded regions indicate $\pm$ one standard deviation.}
    
    \label{fig:delay25_fivepanel_legend}
\end{figure*}

\section{Experimental Results}
\subsection{Benchmarks and Baseline Algorithms}
We evaluate DUPO on deterministic continuous robotic locomotion control MuJoCo \citep{todorov2012mujoco} with random observation delays, including
\textsc{Ant-v4}, \textsc{HalfCheetah-v4}, \textsc{Hopper-v4}, \textsc{Walker2d-v4}, and \textsc{Swimmer-v4}.
We focus on stochastic MDPs that better reflect real-world control systems, in which delays induce state transition
discrepancies and amplify uncertainties. We therefore convert deterministic MuJoCo environments into
stochastic MDPs by adding Gaussian noise $\mathcal N(0,0.1)$ to the environment-executed actions. After constructing stochastic MuJoCo environments, we further introduce random observation delays.
Let $\Delta T$ denote the delay at time $t$, drawn from a uniform distribution where $\Delta T_{\max}$ specifies the maximum delay horizon.
Accordingly, the agent does not directly observe the current state $s_t$; instead, it receives a delayed
message $M_t$, which aggregates the delayed observation together with the associated action history. \\
\noindent The following algorithms were included as baselines in the experiments: DC/AC~\citep{Bouteiller2020RandomDelays}, State Augmentation~\citep{wang2024addressing}, and State Prediction~\citep{wang2024addressing}. DC/AC is a SAC variant that integrates off-policy multi-step value estimation with partial trajectory resampling, yielding strong sample efficiency under random delays. State augmentation concatenates delayed observations with the aligned action history to build a time-calibrated input for standard actor--critic learning. State prediction trains an auxiliary model to infer the delay-free state from the delayed message and
controls using the predicted state at decision time. Detailed experimental settings and hyperparameter configurations are provided in the supplementary material.

\subsection{Results and Discussion}
Table~\ref{tab:mujoco_v4_delay_norm} reports the final return after $1$M global environment steps, summarized by the mean $\mu$ and standard deviation
$\sigma$ over four random seeds.
Since raw returns have different scales across MuJoCo tasks, we normalize results \emph{within each fixed}
$(\Delta T_{\max}, \text{env})$ setting using the DC/AC baseline mean return $\mu_{\text{DC/AC}}$ as a scale factor: $
\tilde{\mu}=\mu/\mu_{\text{DC/AC}}, \qquad
\tilde{\sigma}=\sigma/\mu_{\text{DC/AC}}.
$
Consequently, DC/AC is always $1.00$ in every $(\Delta T_{\max}, \text{env})$ entry, and other numbers can be read as
a multiplicative ratio relative to DC/AC, enabling fair comparison across tasks and delay budgets.\\
\noindent Overall, DUPO achieves the best or near-best normalized mean in most environments, and its advantage becomes more
pronounced as the delay budget increases.
In particular, under severe random delays ($\Delta T_{\max}=25$), DUPO attains strong normalized performance on
HalfCheetah-v4, Hopper-v4, and Walker2d-v4, consistently outperforming
State-Pred and State-Aug, which tend to degrade substantially in the long-delay regime. We also observe that point-estimation baselines can be competitive in short-delay and simpler dynamics (e.g.,
Swimmer-v4 at $\Delta T_{\max}=5$), yet their gains diminish quickly as delays grow, whereas DUPO remains robust.
Figure~\ref{fig:delay25_fivepanel_legend} complements the final scores by illustrating learning dynamics under
$\Delta T_{\max}=25$.
Across five tasks, DUPO typically reaches high-return regimes faster and exhibits more stable training trajectories,
with reduced variance and fewer late-stage plateaus compared to State-Pred and State-Aug. More learning curves are provided in the supplementary material.
A representative case is Walker2d-v4, where long stochastic delays often induce inaccurate delayed-state estimates and
value bias for the baselines, limiting further improvement; in contrast, DUPO continues to improve and sustains higher
returns.
These results align with DUPO's design: modeling the distributional posterior from delayed messages and performing
uncertainty-aware weighting yields more reliable policy updates under multi-modal and high-uncertainty delay
conditions, leading to consistent gains especially in medium-to-long random delays.
\subsection{Ablation Study}
Figure~\ref{fig:curves_two_in_one_col} (left) studies the impact of the observation delay time distribution while keeping the same maximum delay budget $\Delta T_{\max}=10$ on HalfCheetah-v4.
We consider three commonly used delay models: a uniform distribution (Un), a truncated Gaussian distribution (Ga), and a discrete Poisson distribution (Po), all supported on $\{1,\dots,\Delta T_{\max}\}$.
Across these settings, DUPO exhibits consistently strong performance and stable learning dynamics, indicating that the proposed diffusion-based posterior modeling is not tied to a particular delay prior.
Notably, the Gaussian-delay setting yields the fastest improvement and the highest final return, while the uniform and Poisson variants remain competitive but converge more slowly, suggesting that the shape of the delay distribution mainly affects learning speed rather than DUPO's overall robustness.\\
\noindent Figure~\ref{fig:curves_two_in_one_col} (right) further ablates the posterior sampling mechanism used to represent the latent delay-free state.
Our full DUPO uses a diffusion model to produce \emph{multimodal} posterior samples, whereas MD-PO replaces diffusion sampling with Monte Carlo dropout sampling, and G-PO replaces it with a unimodal Gaussian sampler.
The results show a clear performance gap in favor of DUPO: diffusion-based multimodal sampling reaches high-return regimes earlier and plateaus at substantially higher performance than both alternatives.
This highlights the necessity of accurately capturing posterior \emph{multimodality} under stochastic delays; simplified unimodal (Gaussian) or approximate uncertainty sampling (MC dropout) tends to under-represent ambiguous delayed observations, leading to less reliable value estimation and a weaker delayed policy.
\begin{figure}[t]
    \centering
    \begin{minipage}[t]{0.49\columnwidth}
        \centering
        \includegraphics[width=\linewidth]{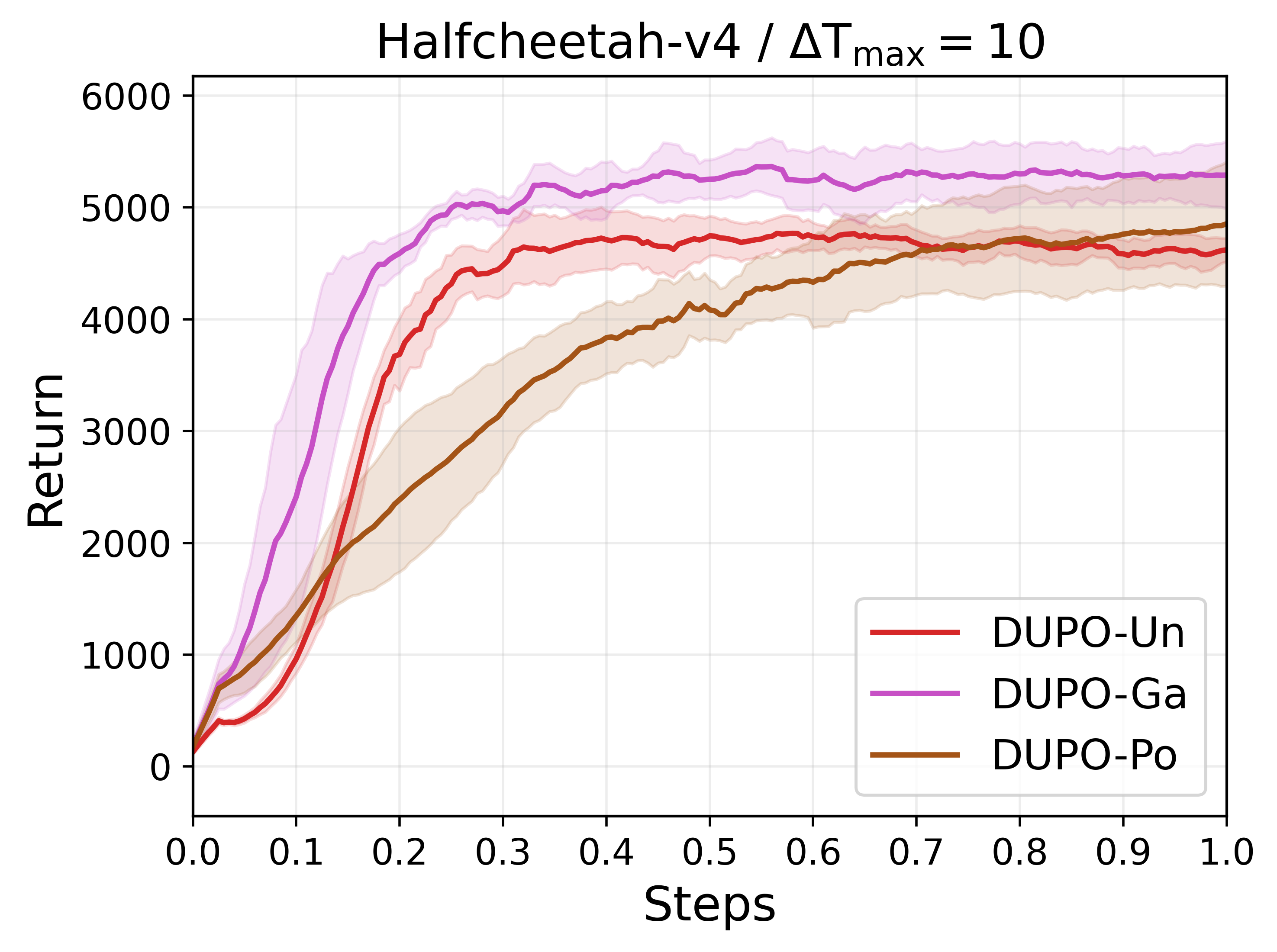}
    \end{minipage}\hfill
    \begin{minipage}[t]{0.49\columnwidth}
        \centering
        \includegraphics[width=\linewidth]{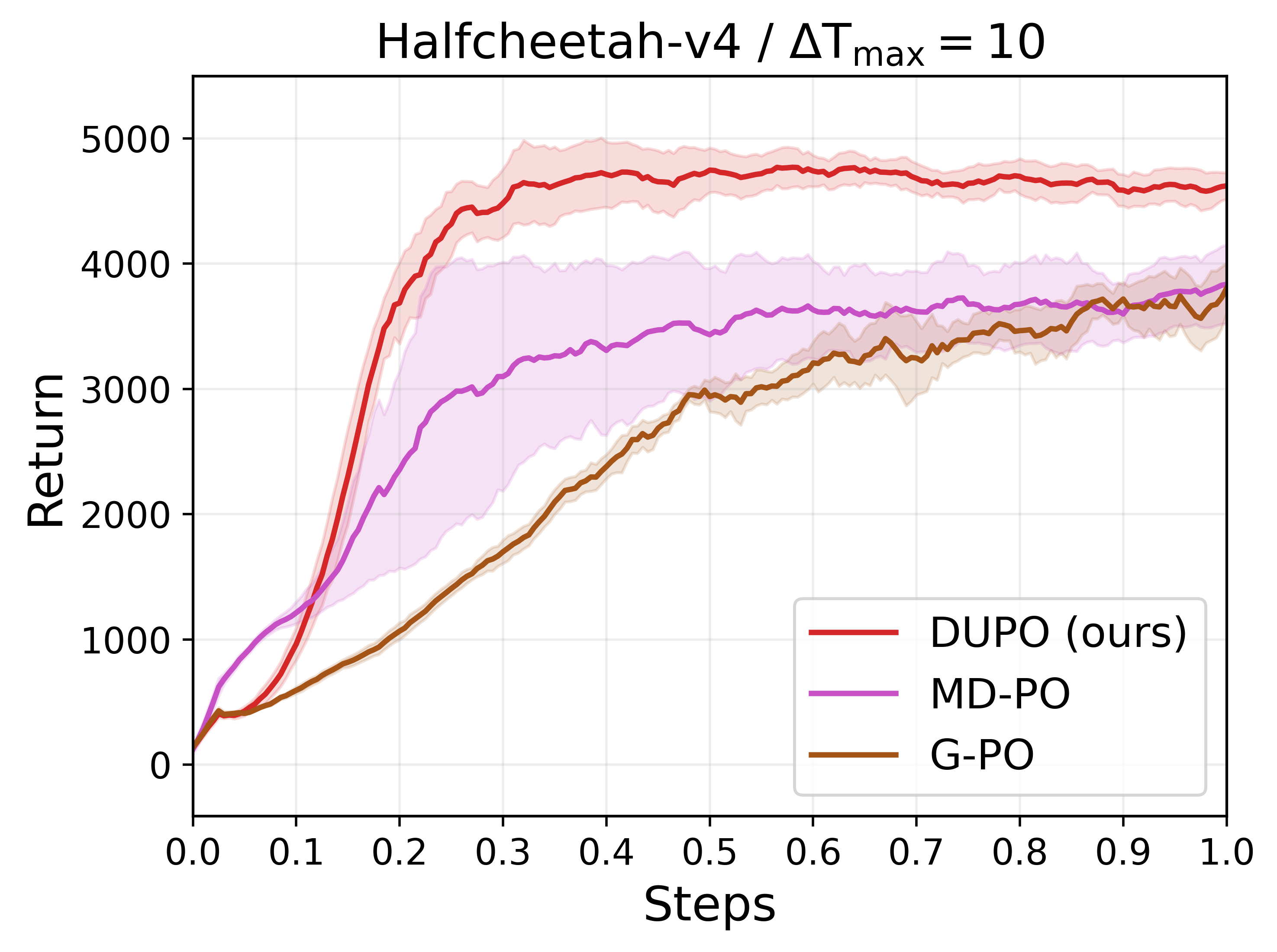}
    \end{minipage}
    \caption{Ablation studies.}
    \label{fig:curves_two_in_one_col}
\end{figure}
\noindent Table~\ref{tab:hopper_dt10_ablation} reports a Hyperparameter sensitivity study on Hopper-v4 with $\Delta T_{\max}=10$, where we normalize returns by the default configuration $(\beta{=}0.8,\,N{=}10)$.
The coefficient $\beta$ controls the uncertainty weight
$\omega(M_t,a)=\beta/(\widehat{\mathrm{Var}}[Q(M_t,a)]+\varepsilon)$, and $N$ denotes the number of posterior samples drawn from the diffusion model to approximate the multimodal latent state distribution.
Varying $\beta$ shows that performance is relatively stable around the default choice, while overly small or large values degrade normalized return, indicating that moderate uncertainty scaling yields the most reliable action weighting.
Increasing $N$ consistently improves performance and reduces variability, suggesting that richer posterior sampling provides a more accurate uncertainty estimate and hence a more effective delayed policy; in particular, larger $N$ achieves the best normalized return under the same delay budget. More details of the ablation studies can be found in the supplementary material.

\begin{table}[t]
\centering
\small
\caption{Hyperparameter sensitivity study on Hopper-v4 with random observation delays $\Delta T_{\max}=10$.}
\label{tab:hopper_dt10_ablation}
\begin{tabular}{c c c c}
\toprule
\multicolumn{2}{c}{Hopper-v4 $\Delta T_{\max}=10$} & \multicolumn{2}{c}{} \\
\cmidrule(lr){1-2}\cmidrule(lr){3-4}
$\beta$ & Return (normalized) & $N$ & Return (normalized) \\
\midrule
0.4 & $0.93\pm0.11$ & 5  & $0.89\pm0.08$ \\
0.8 & $1.00\pm0.03$ & 10 & $1.00\pm0.03$ \\
1.0 & $0.96\pm0.05$ & 15 & $1.02\pm0.02$ \\
1.5 & $1.03\pm0.08$ & 20 & $1.03\pm0.02$ \\
2.0 & $0.91\pm0.04$ & 25 & $1.06\pm0.03$ \\
\bottomrule
\end{tabular}
\end{table}

\section{Conclusion}
In this paper, we studied reinforcement learning under stochastic observation delays and theoretically analyzed how random delays in stochastic MDPs induce an intrinsic discrepancy between delayed observations and the true current state, which fundamentally limits approaches based on single-point state estimation. To address this challenge, we proposed Diffusion-Guided Uncertainty-Aware Delayed Policy Optimization (DUPO), which models the multimodal posterior distribution of the latent delay-free state using a conditional diffusion model and incorporates uncertainty-aware action weighting to stabilize policy updates under delayed state information. By explicitly accounting for delay-induced uncertainty and posterior multimodality, DUPO provides a principled solution to delayed decision-making in stochastic environments. Extensive experiments on continuous-control benchmarks with multiple random delay settings demonstrate that DUPO consistently outperforms strong baselines, with particularly pronounced advantages under medium-to-long delays, validating both the theoretical motivation and practical effectiveness of the proposed approach. To facilitate reproducibility, the code will be publicly released upon acceptance.

\clearpage
\bibliography{dupo_aaai2027}

\ifwithappendix
\onecolumn
\appendix
\section{Theoretic Analysis}
\label{app:theory}

\providecommand{\E}{\mathbb{E}}
\providecommand{\Cov}{\mathrm{Cov}}
\providecommand{\tr}{\mathrm{tr}}

\paragraph{Notation and regularity assumptions.}
All statements below are written for a realized delay value $d=\Delta T$.
We treat the realized delay as part of the delayed message; equivalently, if an implementation stores the delay separately, read $M_t$ below as the pair $(M_t,d)$.
This is natural because the action-history length in $M_t=(s_{t-d},a_{t-d:t-1})$ identifies $d$.
The sigma-field generated by the delayed message is denoted by
\[
\mathcal F_t := \sigma(M_t,d)
 = \sigma\!\left(s_{t-d},a_{t-d:t-1},d\right).
\]
Thus $s_{t-d}$ and the intervening actions are known under $\mathcal F_t$.
For $j=0,\ldots,d$, define
\[
S_j:=s_{t-d+j},\qquad
A_j:=a_{t-d+j},\qquad
\Sigma_j:=\Cov(S_j\mid\mathcal F_t),\qquad
v_j:=\tr(\Sigma_j).
\]
We use the following standard local conditions for stochastic delayed dynamics:
\begin{enumerate}
    \item The latent dynamics along the delayed window satisfy
    $S_{j+1}=f(S_j,A_j)+\xi_j$.
    \item The innovation noise is conditionally centered and conditionally orthogonal to the propagated state uncertainty:
    \[
    \E[\xi_j\mid\mathcal F_t]=0,\qquad
    \tr\!\left(\Cov(f(S_j,A_j),\xi_j\mid\mathcal F_t)\right)=0.
    \]
    \item The transition noise has non-vanishing conditional variance and the deterministic part does not contract conditional uncertainty below a factor $\kappa\ge1$:
    \[
    \tr\!\left(\Cov(\xi_j\mid\mathcal F_t)\right)\ge \nu,\qquad
    \tr\!\left(\Cov(f(S_j,A_j)\mid\mathcal F_t)\right)\ge \kappa\,\tr\!\left(\Cov(S_j\mid\mathcal F_t)\right),
    \]
    for some $\nu>0$.
\end{enumerate}
These assumptions formalize the setting in which stochastic transitions inject fresh uncertainty at every hidden step, while the delayed message cannot fully recover the intervening latent trajectory.

\subsection{Proof of Theorem~\ref{thm:irreducible_delay_discrepancy}}
\label{app:theory:thm1}

\begin{lemma}[Bayes risk under a delayed message]
\label{lem:bayes_risk_delayed}
For any measurable estimator $g(M_t)$,
\begin{equation}
\E\|s_t-g(M_t)\|_2^2
=\E\,\tr\!\big(\Sigma(M_t)\big)+\E\big\|\E[s_t\mid M_t]-g(M_t)\big\|_2^2
\ge \E\,\tr\!\big(\Sigma(M_t)\big).
\label{eq:app_thm1_bayesrisk}
\end{equation}
\end{lemma}

\begin{proof}
Let $\mathcal F=\sigma(M_t)$, $S=s_t$, and $\mu=\E[S\mid\mathcal F]$.
Since $g(M_t)$ is $\mathcal F$-measurable,
\[
S-g(M_t)=(S-\mu)+(\mu-g(M_t)).
\]
Expanding the squared norm and conditioning on $\mathcal F$ yields
\begin{align}
\E\!\left[\|S-g(M_t)\|_2^2\mid\mathcal F\right]
&=\E\!\left[\|S-\mu\|_2^2\mid\mathcal F\right]
  +\|\mu-g(M_t)\|_2^2 \notag\\
&\quad
  +2\,\E[S-\mu\mid\mathcal F]^\top(\mu-g(M_t)).
\end{align}
The cross term is zero because $\E[S-\mu\mid\mathcal F]=0$. Moreover,
\[
\E\!\left[\|S-\mu\|_2^2\mid\mathcal F\right]
=\tr\!\left(\E[(S-\mu)(S-\mu)^\top\mid\mathcal F]\right)
=\tr(\Cov(S\mid\mathcal F)).
\]
Taking expectations gives \eqref{eq:app_thm1_bayesrisk}.
\end{proof}

\begin{lemma}[Delay-amplified conditional variance]
\label{lem:delay_variance_recursion}
For a realized delay $d=\Delta T$,
\begin{equation}
\tr\!\big(\Cov(s_t\mid M_t,d)\big)
\ge
\nu\sum_{i=0}^{d-1}\kappa^i
\ge
\nu d.
\label{eq:app_variance_chain}
\end{equation}
\end{lemma}

\begin{proof}
Because $S_0=s_{t-d}$ is contained in the delayed message, it is $\mathcal F_t$-measurable; hence $v_0=\tr(\Cov(S_0\mid\mathcal F_t))=0$.
For $j=0,\ldots,d-1$, let $X_j=f(S_j,A_j)$.
Using $S_{j+1}=X_j+\xi_j$ and the conditional covariance identity,
\begin{align}
v_{j+1}
&=\tr\!\left(\Cov(X_j+\xi_j\mid\mathcal F_t)\right) \notag\\
&=\tr\!\left(\Cov(X_j\mid\mathcal F_t)\right)
 +\tr\!\left(\Cov(\xi_j\mid\mathcal F_t)\right)
 +2\,\tr\!\left(\Cov(X_j,\xi_j\mid\mathcal F_t)\right) \notag\\
&\ge \kappa v_j+\nu.
\label{eq:app_thm1_rec}
\end{align}
Unrolling the recursion from $v_0=0$ gives
\[
v_d\ge \nu(1+\kappa+\cdots+\kappa^{d-1})
=\nu\sum_{i=0}^{d-1}\kappa^i.
\]
Since $\kappa\ge1$, the final term is at least $\nu d$.
Finally, $S_d=s_t$, which proves \eqref{eq:app_variance_chain}.
\end{proof}

\begin{proof}[Proof of Theorem~\ref{thm:irreducible_delay_discrepancy}]
Lemma~\ref{lem:bayes_risk_delayed} gives the irreducible estimation-risk term:
\begin{equation}
\E\|s_t-g(M_t)\|_2^2
\ge
\E\,\tr\!\big(\Sigma(M_t)\big).
\end{equation}
It remains to lower bound the posterior variance. Since $d$ is included in the delayed message notation, $\Sigma(M_t)=\Cov(s_t\mid M_t,d)$.
Lemma~\ref{lem:delay_variance_recursion} lower bounds this conditional covariance for every realized delay $d$. Hence, conditional on $\Delta T=d$,
\[
\E\,\tr\!\big(\Sigma(M_t)\big)
\ge
\nu\sum_{i=0}^{d-1}\kappa^i
\ge
\nu d.
\]
Replacing $d$ by the realized value $\Delta T$ gives \eqref{eq:irreducible_chain}; averaging over a random delay distribution yields the corresponding unconditional lower bound
$\nu\,\E[\sum_{i=0}^{\Delta T-1}\kappa^i]\ge\nu\,\E[\Delta T]$.
\end{proof}

\subsection{Proof of Theorem~\ref{thm:point_estimate_short}}
\label{app:theory:thm2}

We fix any point estimator $\hat s_t=g(M_t)$ and the induced greedy policy
$\pi_g$ defined by $a_t=a^\star(g(M_t))$, where $a^\star(s):=\arg\max_a Q^\star(s,a)$
(assumed unique) and $V^\star(s):=\max_a Q^\star(s,a)$.
Let the discounted occupancy over $(s,M)$ under a policy $\pi$ be
\[
d^\pi(s,M):=(1-\gamma)\sum_{t=0}^{\infty}\gamma^t\Pr_\pi(s_t=s,M_t=M).
\]

\begin{lemma}[Quadratic growth from $\mu$-strong concavity]
\label{lem:app_qg}
Suppose $Q^\star(s,\cdot)$ is differentiable and $\mu$-strongly concave in $a$, and that $a^\star(s)$ is an unconstrained or interior maximizer. Then, for any $a$,
\begin{equation}
V^\star(s)-Q^\star(s,a)\ \ge\ \frac{\mu}{2}\|a-a^\star(s)\|_2^2.
\label{eq:app_qg}
\end{equation}
\end{lemma}

\begin{proof}
By $\mu$-strong concavity, for any $a$ and $a^\star=a^\star(s)$,
\[
Q^\star(s,a)\le Q^\star(s,a^\star)+\langle \nabla_a Q^\star(s,a^\star),a-a^\star\rangle
-\frac{\mu}{2}\|a-a^\star\|_2^2.
\]
Since $a^\star$ is the unique maximizer, $\nabla_a Q^\star(s,a^\star)=0$, yielding \eqref{eq:app_qg}.
\end{proof}

\begin{lemma}[One-step loss under inverse-Lipschitz greedy map]
\label{lem:app_onestep}
Assume that the greedy action map is locally inverse-Lipschitz on the relevant state region:
\[
\|a^\star(s')-a^\star(s)\|_2\ge L_-\|s'-s\|_2
\qquad\text{for all relevant }s,s'.
\]
Then, for any $(s,M)$ and any point estimator $g(M)$ in this region,
\begin{equation}
V^\star(s)-Q^\star\bigl(s,a^\star(g(M))\bigr)\ \ge\ \frac{\mu L_-^2}{2}\|s-g(M)\|_2^2.
\label{eq:app_onestep}
\end{equation}
\end{lemma}

\begin{proof}
Apply Lemma~\ref{lem:app_qg} with $a=a^\star(g(M))$:
\[
V^\star(s)-Q^\star\bigl(s,a^\star(g(M))\bigr)\ge \frac{\mu}{2}\|a^\star(g(M))-a^\star(s)\|_2^2,
\]
and then use the assumed inverse-Lipschitz property to lower bound the right-hand side by
$\frac{\mu L_-^2}{2}\|g(M)-s\|_2^2$.
\end{proof}

\begin{lemma}[Performance difference with $Q^\star$]
\label{lem:app_pdl}
For any policy $\pi$ (possibly history-dependent) with the same initial distribution as $\pi^\star$,
\begin{equation}
J(\pi^\star)-J(\pi)
=\frac{1}{1-\gamma}\;
\E_{(s,M)\sim d^\pi}\Big[V^\star(s)-Q^\star(s,a)\Big],
\label{eq:app_pdl}
\end{equation}
where $a$ is the action executed by $\pi$ upon observing $M$.
\end{lemma}

\begin{proof}
For the action $a_t$ executed by $\pi$, the Bellman optimality identity gives
\[
Q^\star(s_t,a_t)=\E[r_t+\gamma V^\star(s_{t+1})\mid s_t,a_t].
\]
Therefore,
\[
\E[V^\star(s_t)-Q^\star(s_t,a_t)]
=
\E[V^\star(s_t)-r_t-\gamma V^\star(s_{t+1})].
\]
Multiplying by $\gamma^t$, summing over $t\ge0$, and using telescoping gives
\[
\sum_{t=0}^{\infty}\gamma^t\E[V^\star(s_t)-Q^\star(s_t,a_t)]
=
\E[V^\star(s_0)]-J(\pi).
\]
Because $\pi^\star$ is optimal, $\E[V^\star(s_0)]=J(\pi^\star)$.
Rewriting the left side through the discounted occupancy measure $d^\pi$ proves \eqref{eq:app_pdl}.
\end{proof}

\begin{proof}[Proof of Theorem~\ref{thm:point_estimate_short}]
Applying Lemma~\ref{lem:app_pdl} to $\pi=\pi_g$ and then Lemma~\ref{lem:app_onestep} gives
\begin{equation}
J(\pi^\star)-J(\pi_g)
\ge \frac{\mu L_-^2}{2(1-\gamma)}\;
\E_{(s,M)\sim d^{\pi_g}}\big[\|s-g(M)\|_2^2\big].
\label{eq:app_gap_mse}
\end{equation}
The Bayes-risk decomposition in Lemma~\ref{lem:bayes_risk_delayed}, applied under the occupancy distribution of $\pi_g$, implies
\[
\E_{(s,M)\sim d^{\pi_g}}\big[\|s-g(M)\|_2^2\big]
\ge
\E_{(s,M)\sim d^{\pi_g}}\big[\tr(\Sigma(M))\big].
\]
Under the same delay-variance regularity conditions used in Theorem~\ref{thm:irreducible_delay_discrepancy}, Lemma~\ref{lem:delay_variance_recursion} gives, for realized delay $d=\Delta T$,
\[
\E_{(s,M)\sim d^{\pi_g}}\big[\tr(\Sigma(M))\big]
\ge \nu\sum_{i=0}^{d-1}\kappa^i
\ge \nu d.
\]
Substituting this lower bound into \eqref{eq:app_gap_mse} yields
\[
J(\pi^\star)-J(\pi_g)\ \ge\ \frac{\mu L_-^2\,\nu}{2(1-\gamma)}\,\Delta T,
\]
which proves the claim.
\end{proof}

\section{Implementation Details}
\label{app:impl}

\subsection{Benchmark and Baseline Setting}
\label{subsec:benchmark_baseline}

We evaluate on MuJoCo-v4 continuous-control benchmarks (Ant, HalfCheetah, Hopper, Walker2d, and Swimmer) under random observation delays.
To emulate stochastic control systems, we turn each originally deterministic environment into a stochastic MDP by adding Gaussian noise to the environment-executed actions:
the simulator applies $\tilde a_t = a_t + \epsilon_t$ with $\epsilon_t \sim \mathcal{N}(0,0.1)$.
The observation delay at time $t$ is sampled as an integer-valued random variable
$\Delta T_t \sim \mathrm{Unif}\{1,\ldots,\Delta T_{\max}\}$, and the agent receives a delayed message
$M_t$ that aggregates the delayed observation together with the associated action history, e.g.,
$M_t=(s_{t-\Delta T_t}, a_{t-\Delta T_t:t-1})$.
For each baseline, we implement the method following its original paper and adopt the best-performing hyperparameters recommended by the authors whenever available; when multiple configurations are reported, we select the one that achieves the strongest final performance under the same delay budget.
All baselines are trained with the identical environment construction, interaction budget, and evaluation protocol as DUPO.
We further conduct ablations by varying only the targeted components while keeping the remaining training pipeline unchanged.
\subsection{Ablation Study}
\label{subsec:ablation}

Figure~\ref{fig:curves_two_in_one_col} investigates two factors while keeping the remaining training pipeline unchanged.
In the left panel, we study the effect of the observation-delay time distribution on HalfCheetah-v4 with a fixed maximum delay budget $\Delta T_{\max}=10$.
At each time step, the delay length is an integer random variable $\Delta T_t \in \{1,\ldots,\Delta T_{\max}\}$, and the agent observes a delayed message constructed from the delayed state and the corresponding action history.
We consider three commonly used discrete delay models on this support: an integer-valued uniform distribution (Un), a truncated discrete Gaussian distribution (Ga), and a discrete Poisson distribution truncated to $\{1,\ldots,\Delta T_{\max}\}$ (Po).
Across all three priors, DUPO achieves consistently strong performance and stable learning dynamics, suggesting that diffusion-based posterior modeling is not tied to a particular delay prior.
The Gaussian-delay setting improves fastest and attains the best final return, while the uniform and Poisson variants remain competitive but converge more slowly, indicating that the delay distribution primarily affects learning speed rather than overall robustness.

The right panel ablates how we obtain multiple samples of the latent delay-free state for uncertainty estimation.
Our full DUPO draws \emph{multimodal} posterior samples using the diffusion model.
In contrast, MD-PO (Monte Carlo Dropout Posterior) replaces diffusion sampling with dropout-based sampling: we keep dropout active at inference time and perform multiple stochastic forward passes to generate a set of predicted states.
G-PO (Gaussian Posterior) replaces diffusion sampling with a unimodal Gaussian sampler: we fit (or assume) a single Gaussian distribution over the predicted state and draw multiple state samples from this Gaussian.
Both variants keep the downstream uncertainty-aware weighting and policy optimization unchanged, differing only in the state-sampling mechanism.
The results show a clear advantage for DUPO: diffusion-based multimodal sampling reaches high-return regimes earlier and plateaus at substantially higher performance than MD-PO and G-PO.
This highlights that accurately capturing posterior multimodality under stochastic delays is crucial; unimodal Gaussian sampling or approximate uncertainty sampling via MC dropout tends to under-represent ambiguous delayed observations, leading to less reliable value estimation and weaker delayed policies.

\subsection{Hyperparameters setting}
\begin{table*}[t]
\centering
\small
\caption{Key hyperparameters.}
\label{tab:hparams_key_value}
\setlength{\tabcolsep}{12pt}
\begin{tabular}{l l}
\toprule
\textbf{Key} & \textbf{Value} \\
\midrule
burnin\_num & \texttt{0.25} \\
actor\_history\_merge\_method & \texttt{mlp} \\
batch\_size & \texttt{32} \\
batch\_seq\_len & \texttt{64} \\
max\_epoch & \texttt{200} \\
step\_per\_epoch & \texttt{5000} \\
step\_per\_collect & \texttt{1} \\
update\_per\_step & \texttt{1} \\
gamma & \texttt{0.99} \\
tau & \texttt{0.005} \\
buffer\_size & \texttt{1200000} \\
env\_max\_step & \texttt{5000} \\
\midrule
\multicolumn{2}{l}{\textbf{DUPO-specific}} \\
Diffusion\_num\_samples & \texttt{4} \\
Diffusion\_sample\_chunk\_size & \texttt{2048} \\
Diffusion\_update\_interval & \texttt{30} \\
$\beta$ & \texttt{0.8} \\
$N$ (multimodal sampling times) & \texttt{10} \\
\bottomrule
\end{tabular}
\end{table*}

\section{Learning Curves}
\label{app:curves}

\begin{figure*}[!b]
    \centering
    \setlength{\tabcolsep}{1.2pt}

    \begin{tabular}{@{}ccc@{}}
        \includegraphics[width=0.30\textwidth]{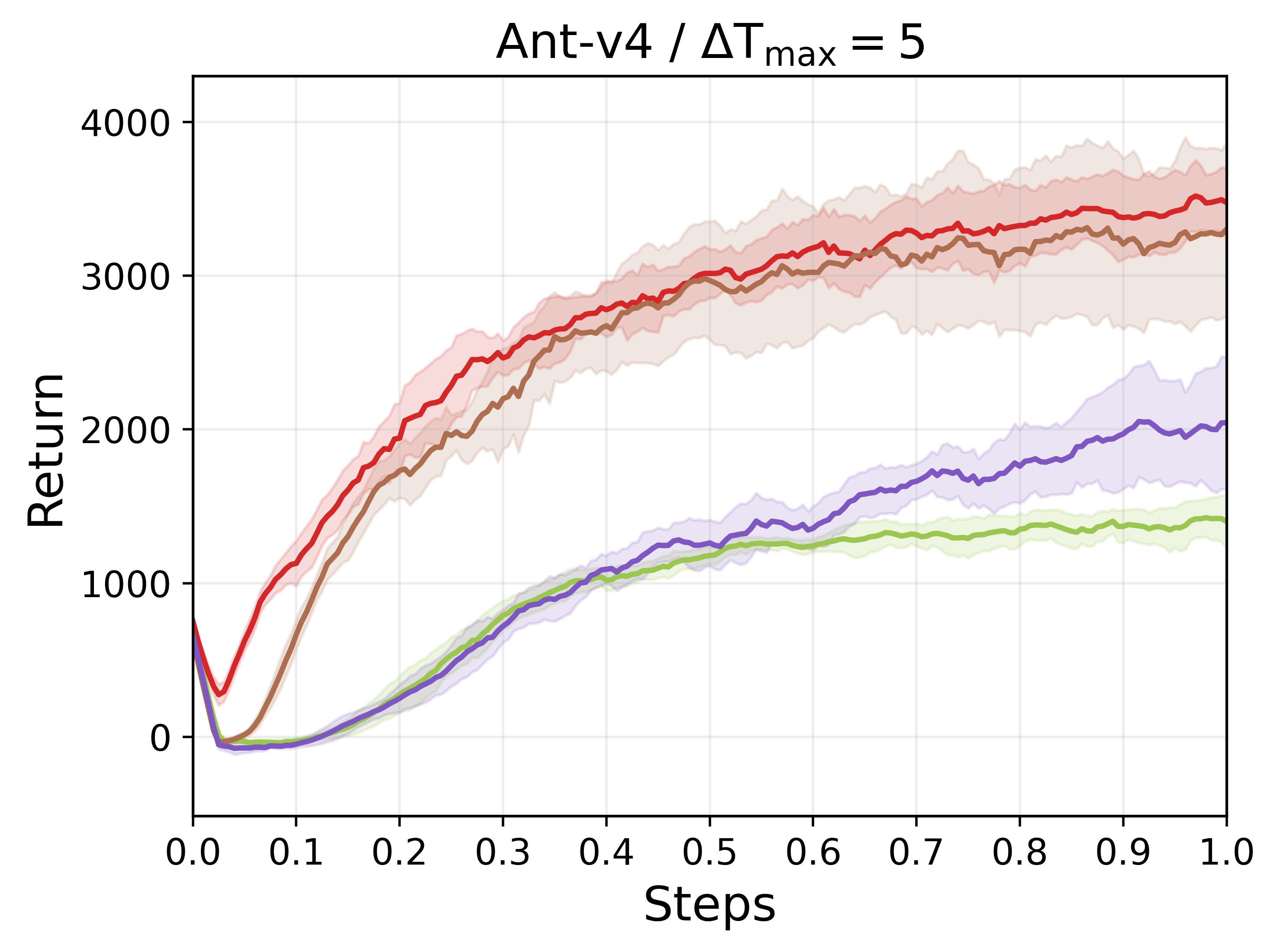} &
        \includegraphics[width=0.30\textwidth]{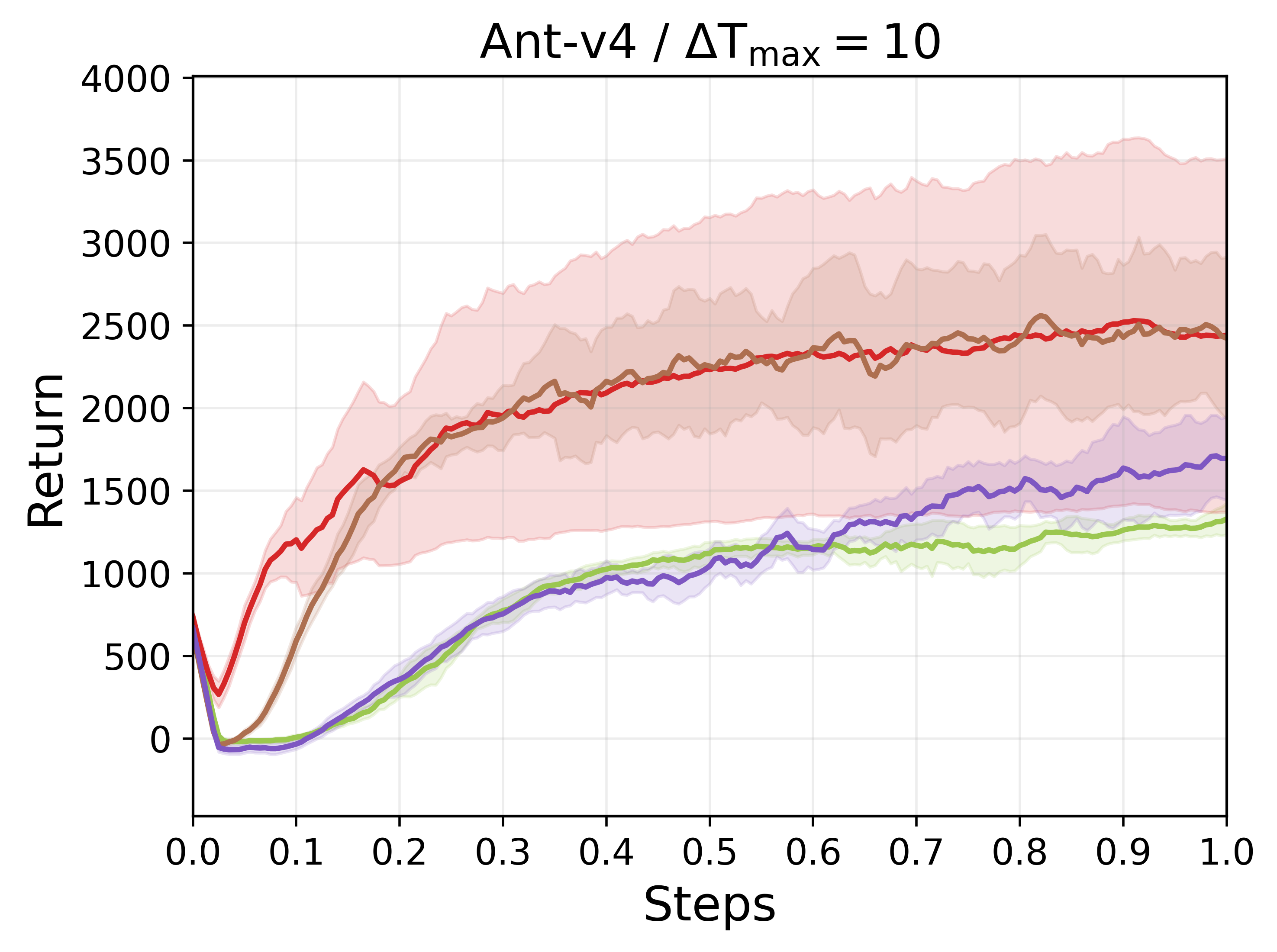} &
        \includegraphics[width=0.30\textwidth]{png/per_project_png/Ant-v4-25_randomdelay_25_ant_20260118_173939.png} \\[-1.4mm]

        \includegraphics[width=0.30\textwidth]{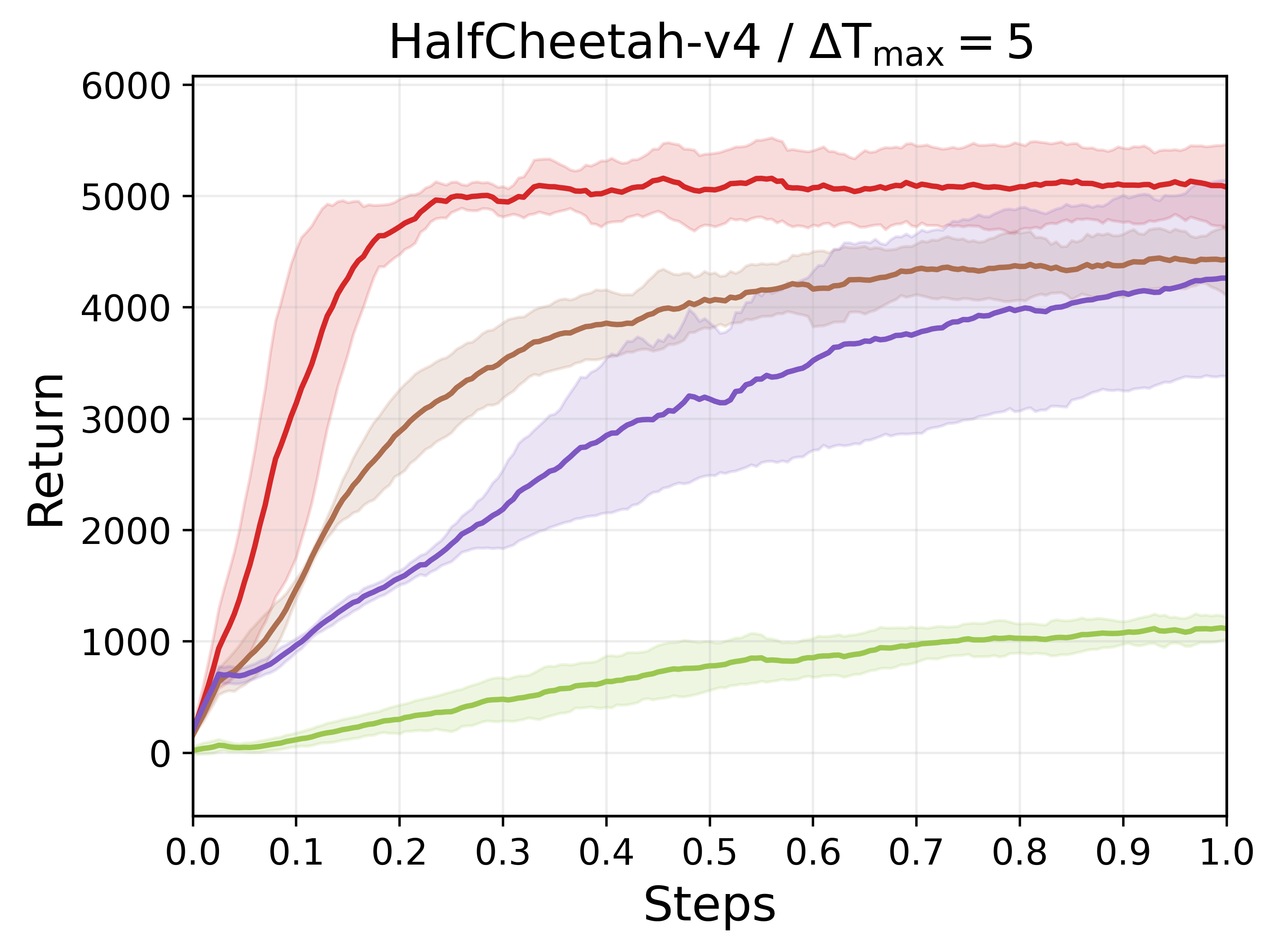} &
        \includegraphics[width=0.30\textwidth]{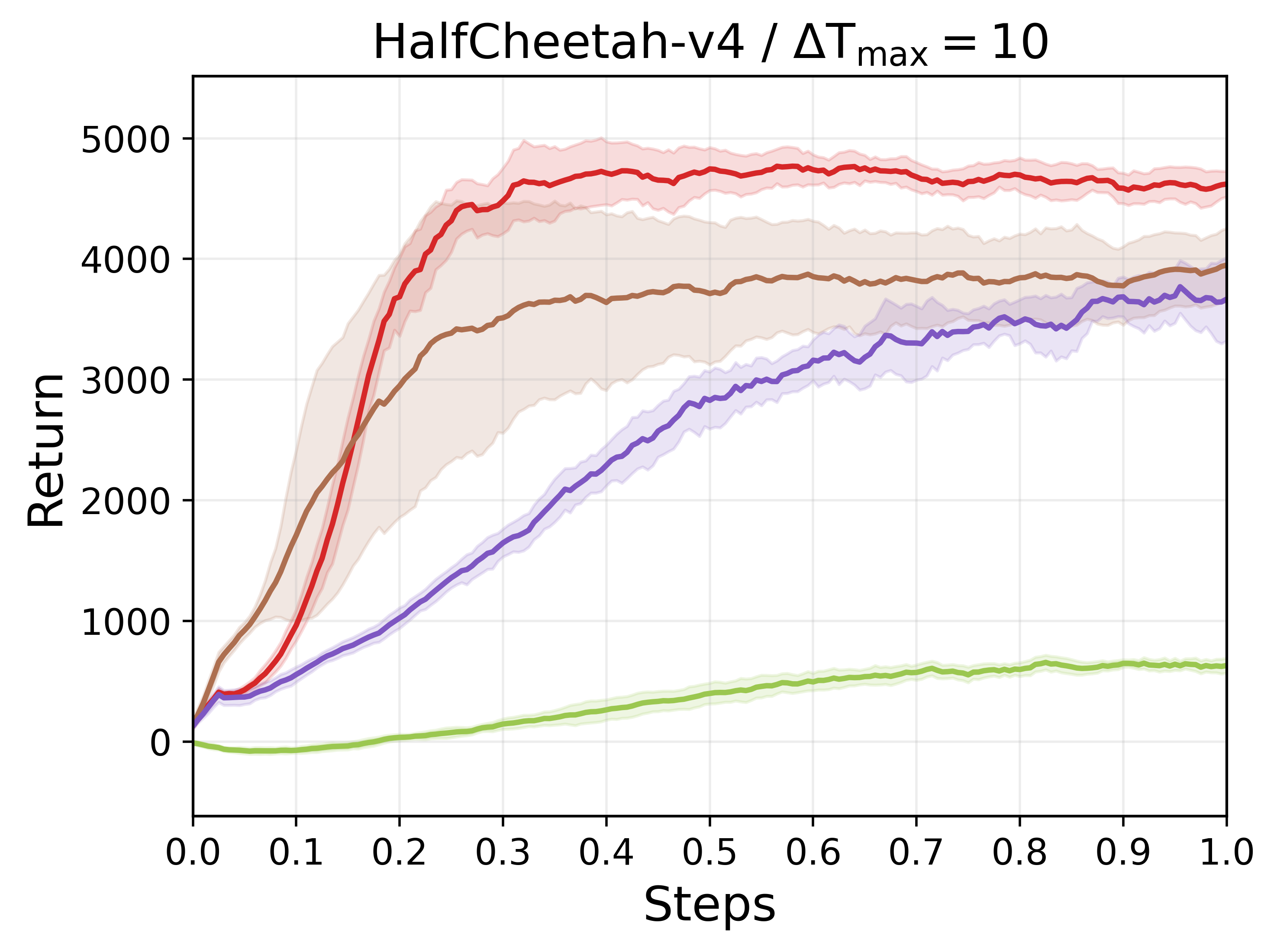} &
        \includegraphics[width=0.30\textwidth]{png/per_project_png/HalfCheetah-v4-25_randomdelay_25_halfcheetah_20260118_173939.png} \\[-1.4mm]

        \includegraphics[width=0.30\textwidth]{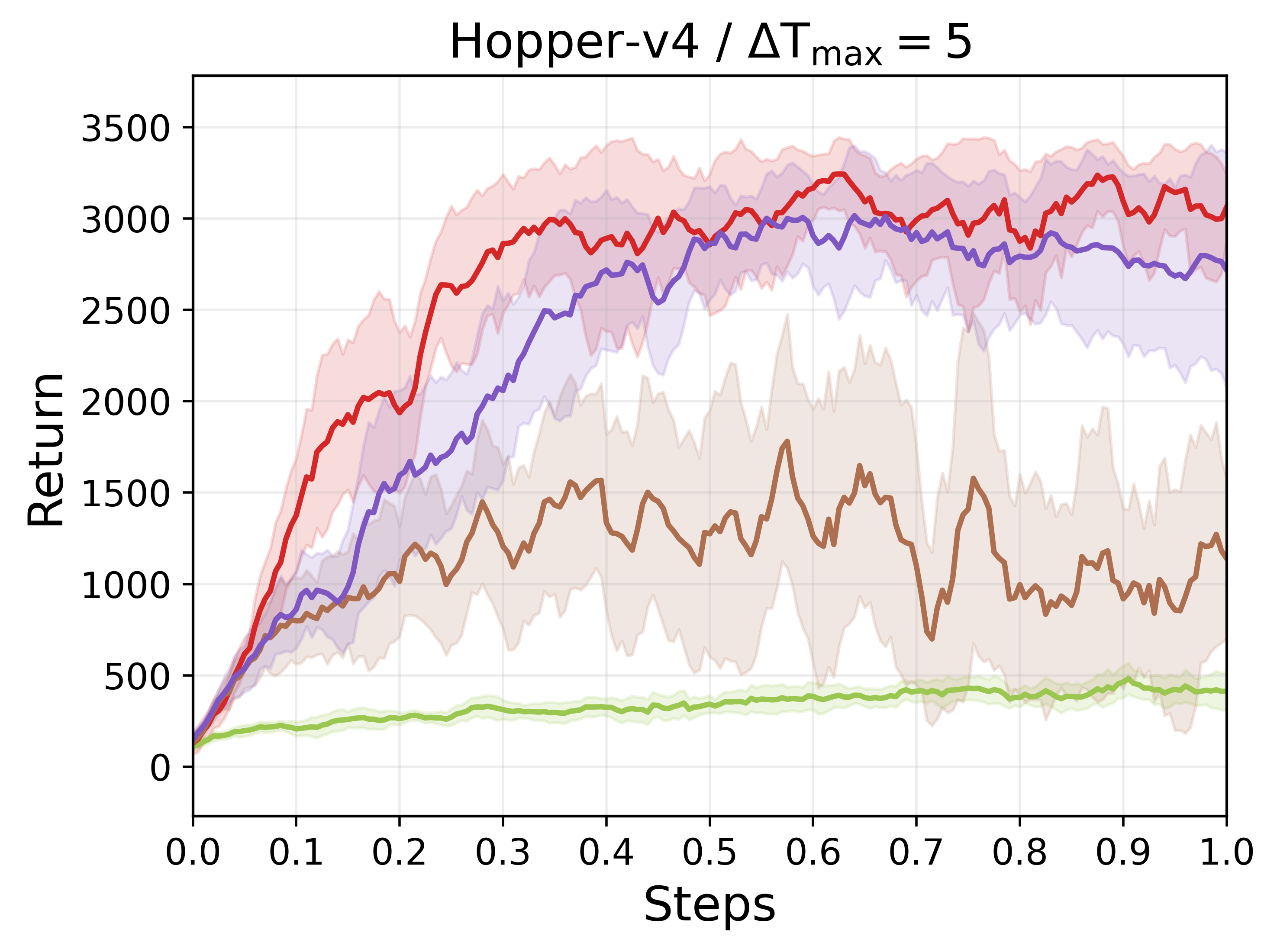} &
        \includegraphics[width=0.30\textwidth]{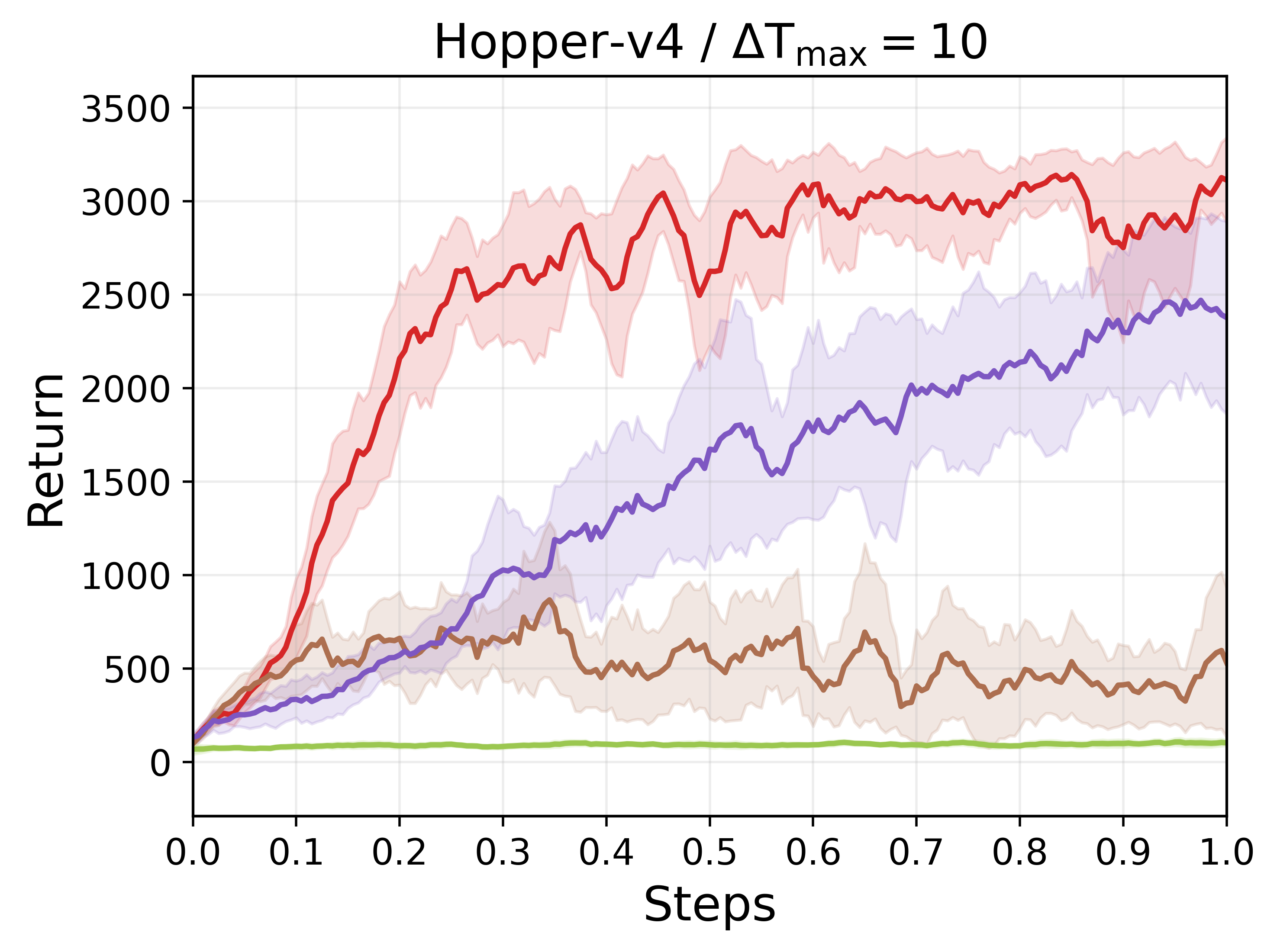} &
        \includegraphics[width=0.30\textwidth]{png/per_project_png/Hopper-v4-25_randomdelay_25_hopper_20260118_173939.png} \\[-1.4mm]

        \includegraphics[width=0.30\textwidth]{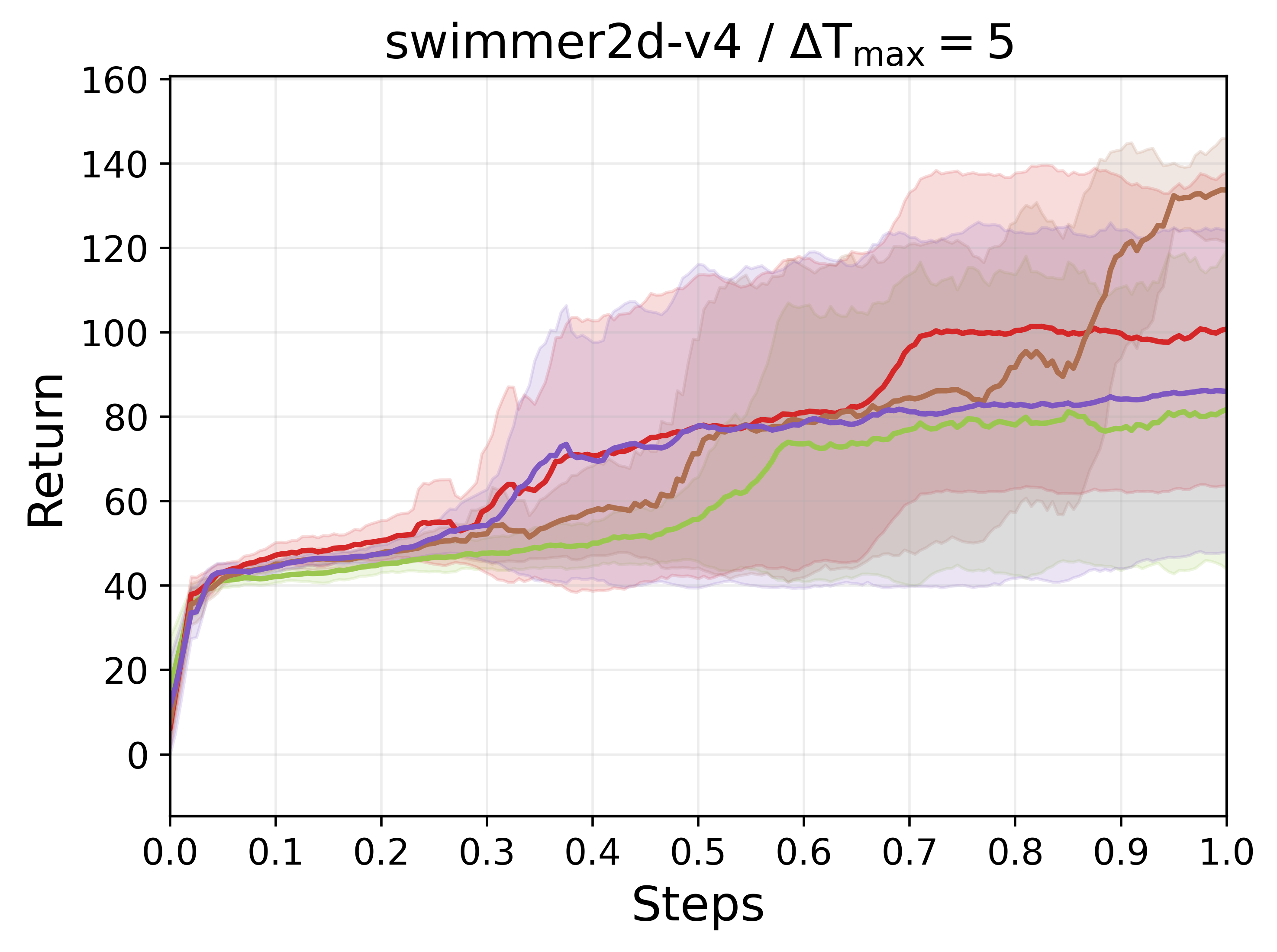} &
        \includegraphics[width=0.30\textwidth]{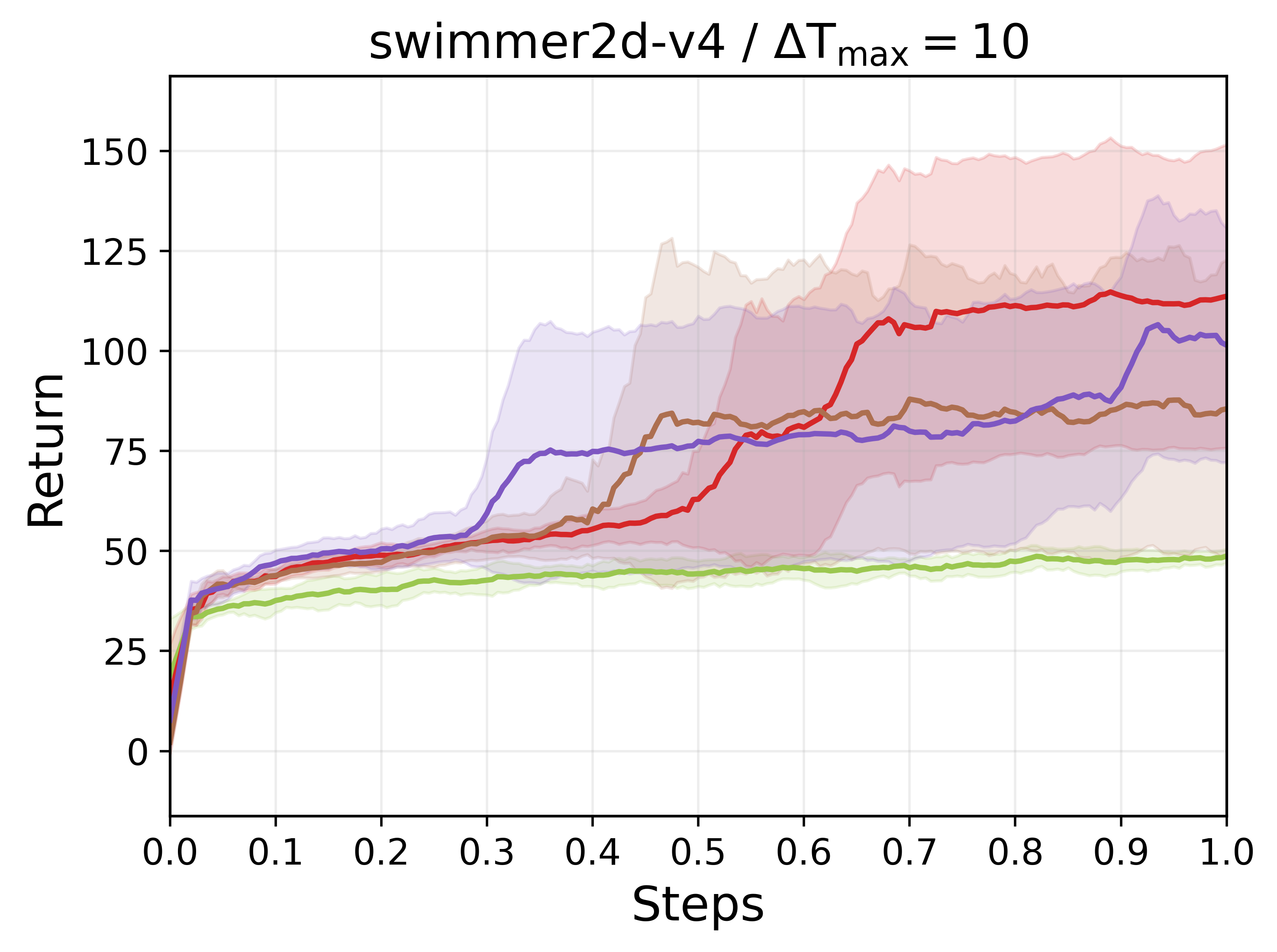} &
        \includegraphics[width=0.30\textwidth]{png/per_project_png/swimmer2d-v4-25_randomdelay_25_swimmer_20260118_173939.png} \\[-1.4mm]

        \includegraphics[width=0.30\textwidth]{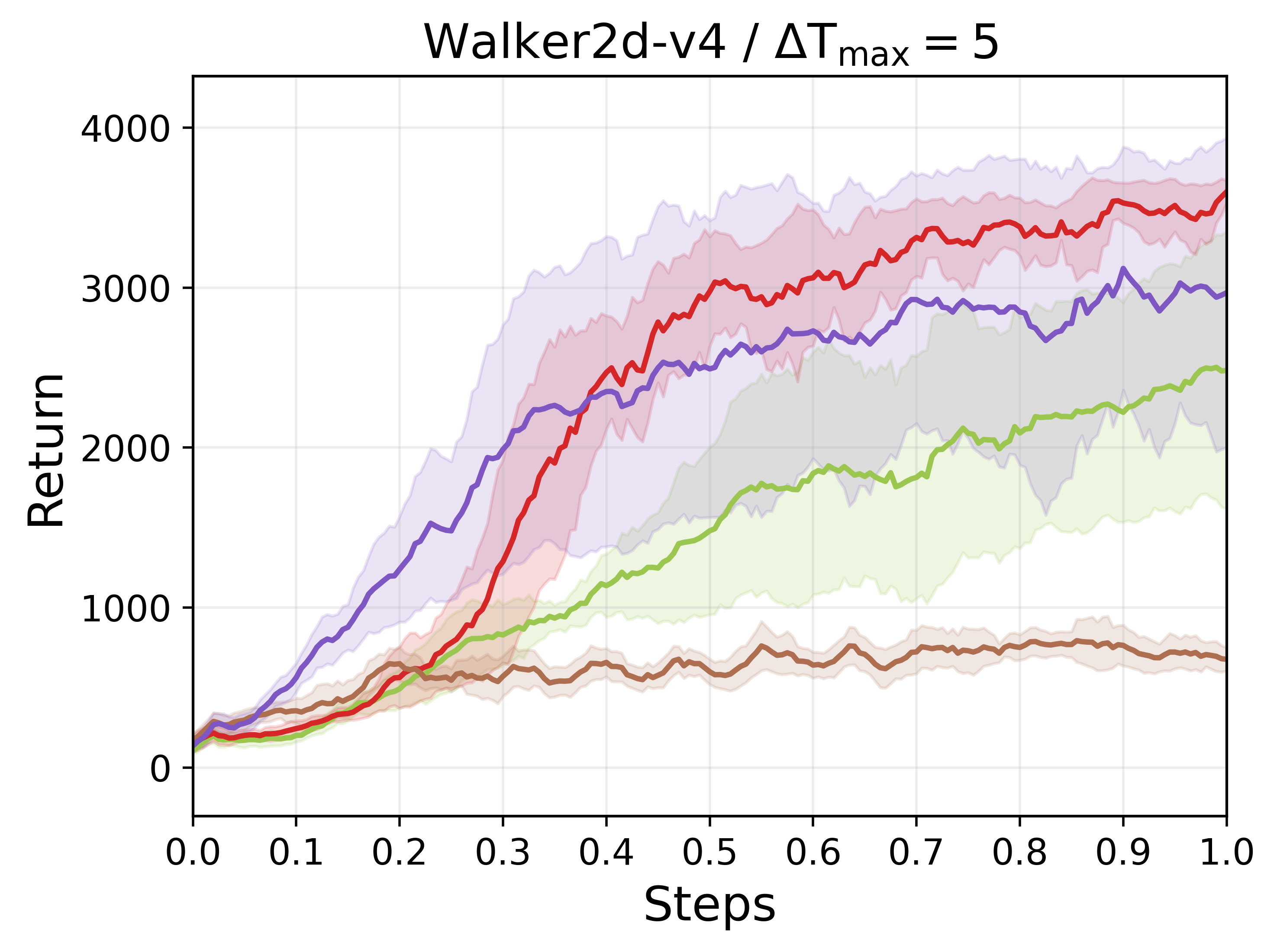} &
        \includegraphics[width=0.30\textwidth]{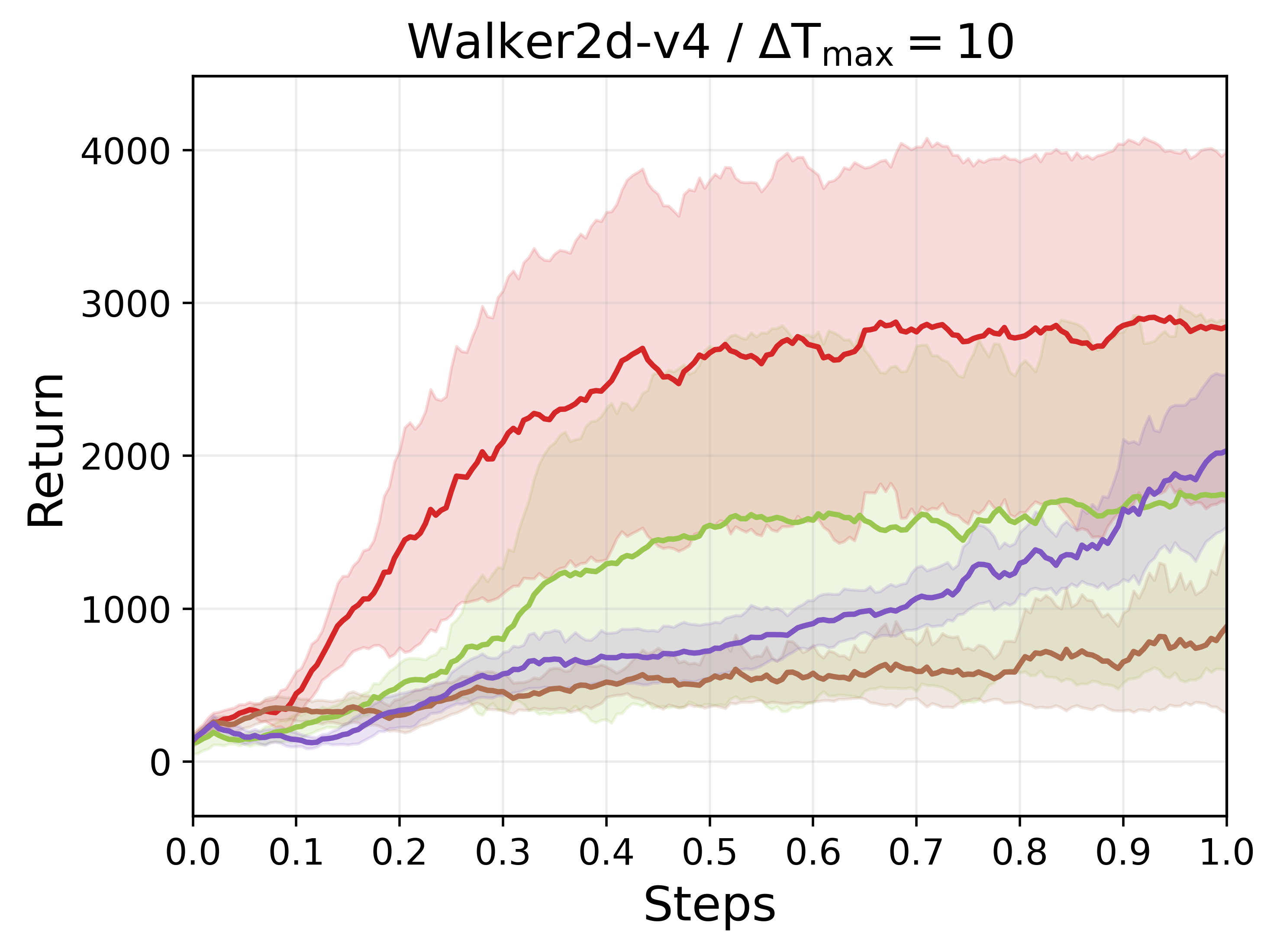} &
        \includegraphics[width=0.30\textwidth]{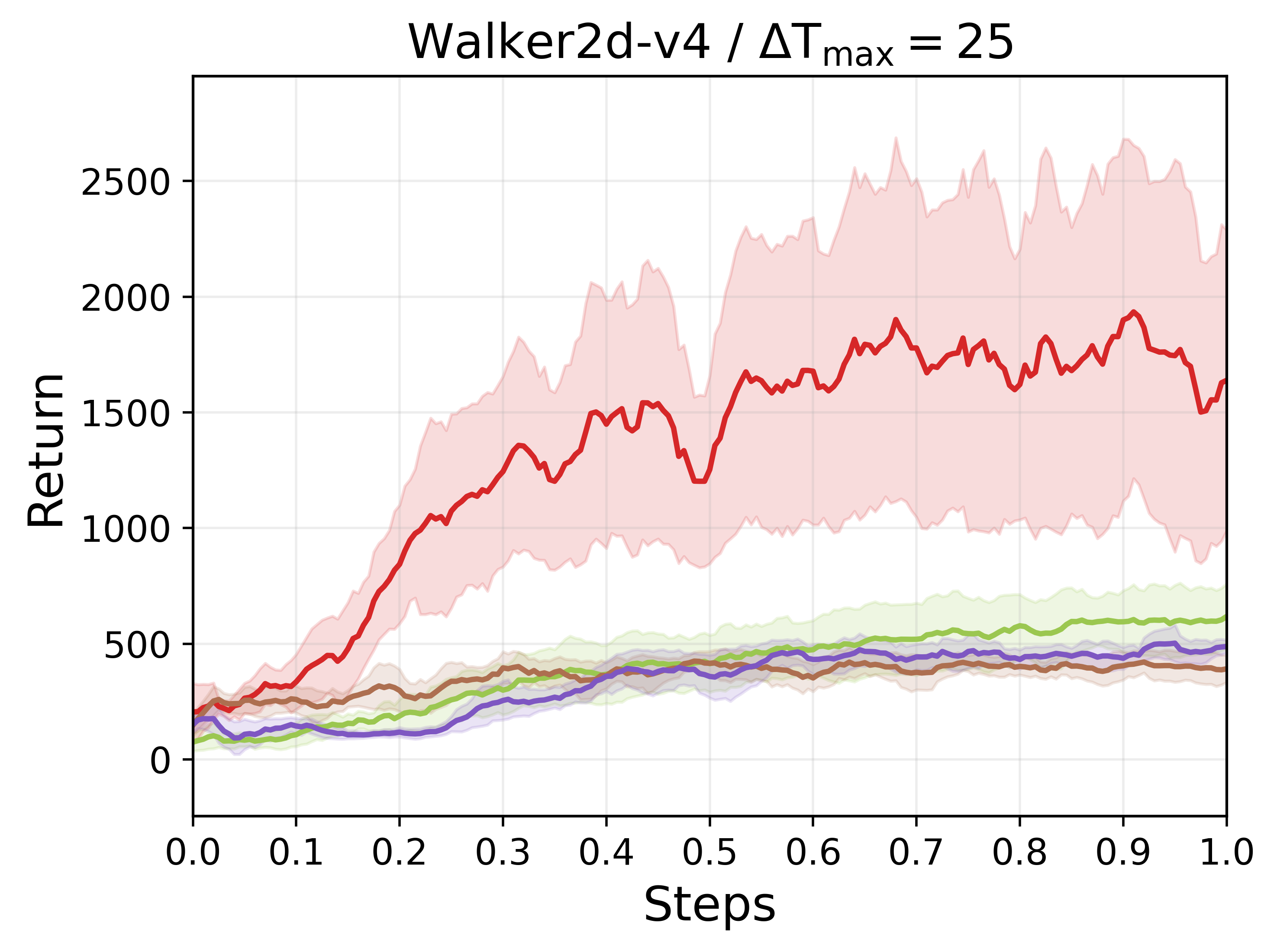} \\
    \end{tabular}

    \includegraphics[width=0.52\textwidth]{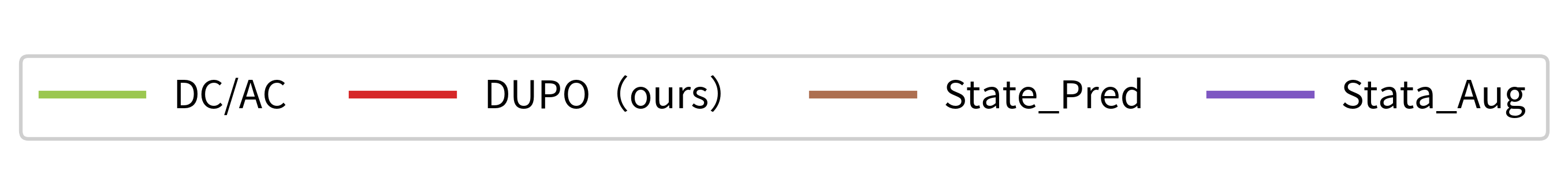}

    \caption{Learning curves under random observation delays with $\Delta T_{\max}\in\{5,10,25\}$.
    Results are averaged over $4$ random seeds; solid lines denote the mean return and shaded regions indicate $\pm 1$ standard deviation.}
    \label{fig:appendix_learning_curves}
\end{figure*}
\fi

\end{document}